\documentclass{article}

\usepackage{microtype}
\usepackage{graphicx}
\usepackage{subfigure}
\usepackage{booktabs} % for professional tables
\usepackage[dvipsnames]{xcolor}

\usepackage{mathtools, amsfonts, amsmath, amsthm}
\usepackage{IEEEtrantools}

\usepackage{hyperref}

\newtheorem{definition}{Definition}
\newtheorem{example}{Example}
\newtheorem{assumption}{Assumption}
\newtheorem{proposition}{Proposition}
\newtheorem{lemma}{Lemma}
\newtheorem{corollary}{Corollary}
\newtheorem{remark}{Remark}
\newtheorem{theorem}{Theorem}

  \newenvironment{proofsketch}{%
  \proof}{\endproof}

\usepackage[accepted]{icml2020}

\def\*#1{\boldsymbol{#1}} %\def\*#1{#1}

\makeatletter
\newcommand{\myitem}[1]{%
\item[#1]\protected@edef\@currentlabel{#1}%
}
\makeatother

\usepackage{xspace}
\newcommand{\acro}[1]{\textsc{#1}\xspace}
\newcommand{\CL}{\acro{CL}}
\newcommand{\ICL}{Idealized \acro{CL}}
\newcommand{\OICL}{optimal Idealized \acro{CL}}

\newcommand{\X}{\mathcal{X}}
\newcommand{\Y}{\mathcal{Y}}
\newcommand{\F}{\mathcal{F}}
\newcommand{\A}{\mathcal{A}}
\newcommand{\C}{\mathcal{C}}
\newcommand{\Ahat}{\widehat{\mathcal{A}}}
\newcommand{\Info}{\mathcal{I}}

\newcommand{\sat}{\acro{Sat}}
\newcommand{\I}{\acro{I}}
\newcommand{\E}{\acro{E}}

\newcommand{\NPC}{\acro{NP-complete}}
\newcommand{\NPH}{\acro{NP-hard}}

\icmltitlerunning{Optimal Continual Learning has Perfect Memory and is \NPH}

\begin{document}

\twocolumn[
%\icmltitle{Optimal Continual Learning: Memory Requirements \& Complexity}
%\icmltitle{Memory Requirements \& Complexity of Optimal Continual Learning}
%\icmltitle{On the Memory \& Complexity of Optimal Continual Learning}
\icmltitle{Optimal Continual Learning has Perfect Memory and is \NPH}
%\icmltitle{Optimal Continual Learning (generally) requires Perfect Memory and is \NPC}



\icmlsetsymbol{equal}{*}

\begin{icmlauthorlist}
\icmlauthor{Jeremias Knoblauch}{X}
\icmlauthor{Hisham Husain}{Y}
\icmlauthor{Tom Diethe}{Z}
\end{icmlauthorlist}

\icmlcorrespondingauthor{Jeremias Knoblauch}{j.knoblauch\@warwick.ac.uk}

\icmlaffiliation{X}{Department of Statistics, Warwick University \& The Alan Turing Institute, London}
\icmlaffiliation{Y}{Australian National University, Canberra \& Data61, Sydney}
\icmlaffiliation{Z}{Amazon}

\icmlkeywords{Continual Learning; Deep Learning Theory; Neural Networks; Complexity}

\vskip 0.3in
]



\printAffiliationsAndNotice{} % otherwise use the standard text.

\begin{abstract}
Continual Learning (\CL) algorithms incrementally learn a predictor or representation across multiple sequentially observed tasks.
%
%Their goal is to incrementally  for all previously observed tasks.
%
Designing \CL algorithms that perform reliably and avoid so-called \textit{catastrophic forgetting} has proven a persistent challenge.
The current paper develops a theoretical approach that explains why.
%
%To this end, optimality criteria are used as a formal way  of assessing whether a given \CL algorithm
%
In particular, we derive the computational properties which \CL algorithms would have to possess in order not to suffer from catastrophic forgetting.
%
%
%
%In particular, our two main findings shed light on their memory and computation requirements.
%
Our main finding is that such \textit{optimal} \CL algorithms generally solve an \NPH problem and will require perfect memory to do so.
The findings are of theoretical interest, but also explain the excellent performance of \CL algorithms using experience replay, episodic memory and core sets relative to regularization-based approaches.
%
%Further, they sheds light on promising directions of future research on \CL. 
%
\end{abstract}

\section{Introduction}
\label{sec:introduction}

Continual Learning (\CL) is a machine learning paradigm which takes inspiration from the ways in which biological organisms learn in the real world: Rather than observing a set of independent and identically distributed observations, \CL seeks to design algorithms that sequentially learn from observations corresponding to different tasks.
Unlike biological organisms, the artificial neural networks used for solving this problem suffer from \textit{catastropic forgetting} \citep{CatastrophicForgettingOriginal}.
Simply put, this phenomenon describes that sequentially learning on an increasing number of  tasks will eventually yield increasingly poor representations of and predictions on previously observed tasks.

\CL algorithms are not only a challenging research topic, they are also of tremendous practical importance \citep[see e.g.][]{CLInPractice}: Often, it is impractical or even impossible to re-train a model every time new data arrives.
For example, data may be too sensitive or expensive to store long term.
Even if the storage of data is not a problem, increasingly complex models may make re-training computationally prohibitive.
%
%While ongoing research has constantly improved  algorithms for \CL, the problem nonetheless has proven a persistently difficult challenge. 

%Again in analogy to the challenges faced by biological brains, these clusters of similar observations are referred to as \textit{tasks}.
%

%
While algorithms tackling the issue have steadily improved over the last few years, the \CL problem has remained a persistently difficult challenge. 
In response, a growing body of research has  designed novel \CL algorithms based on different paradigms.
%While ongoing research has constantly improved  algorithms for \CL, the problem nonetheless has proven a persistently difficult challenge. 
%
In an attempt to structure the research output of the field, several paper have sought to classify these different paradigms \citep{CLReview, GalUnification, threeScenarios}.
Broadly speaking, one can divide existing \CL algorithms into one of three families: Regularization-based approaches \citep[e.g.][]{EWC, synapticIntelligence, progressAndCompress, EWCLaplace, RiemannianWalk, FineGrainedCL}, replay-based approaches \citep{deepGenerativeReplay, GradientEpisodicMemory, DeepGenerativeCL, ExperienceReplay} as well as Bayesian and variationally Bayesian approaches \citep[e.g.][]{VariationalCL, naturalVCL, ContinualGP, TitsiasCL}.
Other successful methods include learning a set of new parameters per task without discarding previously learnt ones \citep[e.g.][]{ProgressiveNNs} as well as methods inspired by nearest-neighbour type considerations \citep[e.g.][]{icarl}.

For the purposes of this paper, it is important to contrast the \CL setting with more traditional sequential learning paradigms.
The idea of learning inputs sequentially is all but new---in fact, \citet{Lauritzen1880} dates its origins as early as the 1880s.
In clear contrast to the \CL setting however, more traditional approaches rely on dependence assumptions between consecutive observations and tasks.
For example, the work of \citet{OnlineBayes1}  introduces an  approximately Bayesian algorithm for neural networks that can be updated continuously in an on-line fashion.
In clear opposition to the \CL paradigm, the algorithm assumes that all observations were generated by the same and correctly specified data generating mechanism.
Though restrictive, this also enables the authors to demonstrate asymptotic efficiency.
Similarly, the contributions to sequential and streaming learning made by \citet{OnlineBayes2}, \citet{StreamingVB} or \citet{PopulationPosteriorStreaming} are geared towards streaming settings in which the differences in observations vary only very slowly. 
Clearly, these methods are feasible only on a restricted range of problem settings. While this makes them less universally applicable, it also makes them reliable and their failures more interpretable.
%the similarity between sequentially arriving tasks is relatively large.
%
In clear distinction to this, \CL algorithms seek to mimic biological brains and thus do not impose assumptions about the dependence between observations and tasks.

This has led to methods that can produce reasonable results over a wider range of settings, but also implies that \CL algorithms generally have no clear performance guarantees.
Worse still, it often leaves open in which contexts these algorithms can be expected to function reliably \citep[see e.g.][]{threeScenarios}.
%
%This raises two important questions: 
%
%Is it possible to produce \CL algorithms that can function reliably without making assumptions on the data generating mechanisms? 
%
%And if not, 
%
Does this mean that it is generally computationally impossible to produce reliable \CL algorithms without making assumptions on the task distributions? 
In this paper, we develop theoretical arguments which confirm this suspicion: Optimal \CL algorithms would have to solve an \NPH problem and perfectly memorize the past.
This is not a purely negative result, however: The requirement of perfect memory in particular is of fundamental practical interest. 
Specifically, it explains recent results that favour approaches based on replay, episodic memory and core sets relative to regularization-based approaches \citep[see e.g.][]{VariationalCL, threeScenarios}.
Thus, the findings also show that methods (approximately) recalling or reconstructing observations from previously observed tasks will be the most promising in developing reliable \CL algorithms.

In this paper, we introduce a definition of \CL  wide enough to encompass the large number of competing approaches introduced over the last few years.
Further, we define an equally flexible notion of \textit{optimality} for \CL algorithms.
In the context of this paper, these optimality criteria are used to rigorously define what catastrophic forgetting entails: 
A \CL algorithm which is optimal with respect to a given criterion avoids catastrophic forgetting (as formalized by this criterion).
We then ask a central question:
\begin{quote}
    \textit{What are the  computational properties of an optimal \CL algorithm?}
\end{quote}

In answering this question, we develop new insights into the design requirements for  \CL algorithms that avoid catastrophic forgetting and provide the first thoroughly theoretical treatment of \CL algorithms:
\begin{itemize}
    \myitem{(1)} 
    We show that without any further assumptions, the optimality of \CL algorithms can be studied with the tools of set theory (Lemma \ref{lemma:OCL=OICL}), which drastically simplifies the subsequent analysis.
    \myitem{(2)}
    We show that optimal \CL algorithms can solve a version of the set intersection decision problem (Lemma \ref{lemma:decision_problem}). Crucially, this decision problem will generally be \NPC, meaning that the optimal \CL algorithm itself is \NPH (Theorem \ref{thm:NP}; Corollary \ref{corollary:NPH}).
    \myitem{(3)}
    We define the notion of an equivalence sets and use their properties (Lemma \ref{lemma:equivalence_set_properties}) to motivate the definition of perfect memory: Specifically, we say that a \CL algorithm has perfect memory if it stores at least one element from each equivalence set (Definition \ref{definition:perfect_memory}).
    \myitem{(4)}
    Re-using the decision problem of Lemma \ref{lemma:decision_problem}, we show that optimal \CL 
    has perfect memory under mild regularity conditions (Theorem \ref{thm:optimal_CL_perfect_memory}; Corollary \ref{corollary:optimal_CL_perfect_memory}).
\end{itemize}

Our findings illuminate that \CL algorithms can be seen as polynomial time heuristics targeted at solving an \NPH problem.
Further, they explain why mimicking the perfect memory requirement of optimal \CL through  memorization heuristics generally outperforms \CL algorithms based on regularization heuristics.
%Throughout, we will give a high-level overview of the theoretical arguments and pictorially illustrate the logic of the arguments. 
%
Throughout, we make an effort to provide proof sketches for all of the most important findings. Detailed derivations for all claims can be found in the Appendix.

%
%which can be understood as smallest common denominator across a range of competing approaches. 
%
%Next, we investigate which computational properties an optimal \CL algorithm would have to possess. 
%In particular, for input space $\mathcal{X}$ and output space $\mathcal{Y}$ and function class $\mathcal{F} \subset \mathcal{Y}^{\mathcal{X}}$ from which one infers $f \in \mathcal{F}$, then even  if there exists $f$ so that $f(x) = y$ gives the correct output $y \in \mathcal{Y}$ for each input $x \in \mathcal{X}$, \CL without forgettingis impossible.

%\X{Some more detail: Explain decision problem, why we do what we do etc.}

%While the main results are of theoretical nature, they are also of interest for more applied research.
%
%In particular, they explain recent benchmarking results \X{cite comparisons} and propose strategies for the design of future \CL algorithms.
%

The remainder of this paper is structured as follows: 
In Section \ref{sec:prelim}, we introduce basic notation and concepts. 
Next, we define \CL algorithms and their optimality in Section \ref{sec:optimal_CL}.
%
%In the context of this paper, these optimality criteria formalize whether or not a given \CL algorithm catastrophically forgets (as judged by the criterion).
%
%Accordingly, any \CL algorithm that does not suffer catastrophic forgetting (as judged by a given criterion) is optimal.
%
We then show that any optimal \CL algorithm can be expressed in an idealized alternative form in Section \ref{sec:idealization}. 
This idealized form is interesting because it drastically simplifies the analysis.
Fourth, we show that an optimal \CL algorithm can solve a set intersection decision problem. Under mild conditions, this decision problem is \NPC, which we use to prove that the corresponding optimization problem  (i.e., optimal \CL) is \NPH in Section \ref{sec:NP}. 
Fifth, in Section \ref{sec:memory} we define a notion of perfect memory that is suitable for \CL algorithms. We then demonstrate that optimal \CL algorithms will generally have perfect memory.
Lastly, Section \ref{sec:consequences_for_CL_design} provides a brief discussion of the implications of our results for \CL algorithm design.

\section{Preliminaries}
\label{sec:prelim}

Throughout, we deal with random variables $\*X_t, \*Y_t$. 
Realizations of the random variable $\*X_t$ live on the input space $\X$ and provide information about the random outputs $\*Y_t$ with realizations on the output space $\Y$.
Throughout, $\mathcal{P}(A)$ denotes the collection of all probability measures on $A$.

\begin{definition}[Tasks]
    For a number $T \in \mathbb{N}$ and random variables $\{(\*X_t, \*Y_t)\}_{t=1}^T$ defined on the same spaces $\X$ and $\Y$, 
    the random variable $(\*X_t, \*Y_t)$ is the $t$-th task, and its probability space is $(\X_t \times \Y_t, \*\Sigma, \mathbb{P}_{t})$, where $\*\Sigma$ is a $\sigma$-algebra and $\mathbb{P}_{t}$ a probability measure on $\X_t \times \Y_t \subseteq \X \times \Y$.
    %  For a number $T \in \mathbb{N}$, random variables $\{(\*X_t, \*Y_t)\}_{t=1}^T$ and an indicator function $I:\X\times\Y \to 2^{\{1,2,\dots T\}}$,  
    %  %
    %  \begin{IEEEeqnarray}{rCl}
    %      (\*X, \*Y) & = & \left(\sum_{t=1}^T 1_{t \in I(\*X, \*Y)}\*X_t, \sum_{t=1}^T 1_{t \in I(\*X, \*Y)}\*Y_t\right)
    %      \nonumber
    %  \end{IEEEeqnarray}
    %  %
    %  is the joint distribution of $\*X$ and $\*Y$. The random variable $(\*X_t, \*Y_t)$ is the $t$-th task, and we its probability space as $(\X_t \times \Y_t, \*\Sigma, \mathbb{P}_{t})$, where $\X_t \times \Y_t \subset \X \times \Y$.
\end{definition}

%\X{Give shitty example and pic}

%\Z{Here we can give examples that correspond to split (and permuted tasks) too (or possibly link to appendix). To me, split is $\X_1 = [0, 4]$, $\X_2 = [0,4]$ and $\Y_1 = \{1,2\}$, $\Y_1 = \{3,4\}$. Permuted is the weird one as we know, where each input space $\X_t$ is formed of different permutations of the product space $\X_1$, which I wrote in my other doc as $\X_t = \alpha_t(\X)$ where $\alpha_t$ is a bijective map from $\X$ to $\X$. Also note that the full space $\X$, which is the collection of all permutations of $\X_1$ including itself, forms a group called the \emph{symmetric group} $S_n$ over product spaces of $\X$. The order of the group is $|S_n| = n!$, which naturally places a bound on the number of tasks that can be produced with this method.}

Given a sequence of samples from task-specific random variables, a \CL algorithm sequentially  learns a predictor for $\*Y_t$ given $\*X_t$.
This means that there will be some hypothesis class $\F_{\*\Theta}$ consisting of conditional distributions which allow (probabilistic) predictions about likely values of $\*Y_t$.

\begin{definition}[\CL hypothesis class]
     The \CL hypothesis class $\F_{\*\Theta}$ is parameterized by $\*\Theta$: For any $f\in\F_{\*\Theta}$, there exists a $\*\theta \in \*\Theta$ so that $f_{\*\theta} = f$.
     More precisely, $\F_{\*\Theta} \subset \mathcal{P}(\Y)^{\X}$ if the task label is not conditioned on. 
     Alternatively, $\F_{\*\Theta} \subset \mathcal{P}(\Y)^{\X \times \{1,2,\dots T\}}$ if the label is conditioned on.
\end{definition}
\begin{remark}
Note that while this formulation may seem to exclude Bayesian approaches at first glance, this is not the case.
In fact, one simply notes that a posterior distribution acts as an (infinite-dimensional) parameter:
Specifically, suppose we have some model $m_{\*\kappa}$ parameterized by a finite-dimensional parameter $\*\kappa \in \*K$ and want to form a posterior belief about it.
In this case, we could recover the Bayesian approach by setting $\*\Theta = \mathcal{P}(\*K)$ to be the collection of possible posteriors and $f_{\*\theta}(x) = \int_{\*K}m_{\*\kappa}(x)d\*\theta(\*\kappa)$.
\end{remark}
%
% \begin{remark}
%     If the task label $t$ is conditioned on, $f_{\*\theta} \in \F_{\*\Theta}$ is some parameterized conditional distribution given by $f_{\*\theta} = \{p_{\*\theta}\left(\*Y_t|x,  t\right)\}_{x \in \X_t}$.
%     %
%     Alternatively if the task label is not conditioned on, $f_{\*\theta} = \{p_{\*\theta}\left(\*Y_t|x\right)\}_{x \in \X_t}$.
% \end{remark}
%
\begin{remark}
    The set-valued indicator functions $1_{y} = 1_{\{y\}}$ are elements of $\mathcal{P}(\Y)$ for all $y \in \Y$. Thus, $\Y^{\X} \subset \mathcal{P}(\Y)^{\X}$ and $\Y^{\X\times \{1,2,\dots T\}} \subset \mathcal{P}(\Y)^{\X\times \{1,2,\dots T\}}$. In other words, defining the hypothesis class as conditional distributions also recovers deterministic input-output mappings (via degenerate conditional distributions).
    For example, one may choose $h \in \Y^{\X}$ (or $h \in \Y^{\X \times \{1,2,\dots T\}}$) and construct $p(\*Y|x) = 1_{h(x) = y}\cdot y$ (or $p(\*Y|x, t) = 1_{h(x,t) = y}\cdot y$). 
    \label{remark:deterministic_mappings}
\end{remark} 
% \Y{I would say just be a little careful here because whilst I understand you are mentioning that functions are embedded in the space of distributions (via degenerate distributions), stating that $1_{y}$ is an element of $\mathcal{P}(\Y)$ is not right unless these aren't the indicator functions but the set-valued indicator function since elements of $\mathcal{P}(\Y)$ are set valued functions (probability measures) so may be a little misleading. }
%
\begin{remark}
    All results derived and definitions provided in the remainder can be modified in obvious ways to account for the case where $\F_{\*\Theta} \subseteq \mathcal{P}(\Y)^{\X \times \{1,2,\dots T\}}$. To keep notation as simple as possible however, we will assume that $\F_{\*\Theta} \subseteq \mathcal{P}(\Y)^{\X}$. 
    %This means that $f_{\*\theta}(x) \in \mathcal{P}(\Y)$ and that $f_{\*\theta}(x)(y) = f_{\*\theta}(x,y) = p_{\*\theta}(x|y)$.
\end{remark}

\section{Continual Learning \& Optimality}
\label{sec:optimal_CL}

Having defined both tasks and hypothesis classes, we now define the collection of procedures that constitute \CL algorithms.
To the best of our knowledge, this definition is wide enough to encompass any existing \CL algorithm.
Figure \ref{Fig:CL_MNIST} visualizes this definition.

\begin{definition}[Continual Learning]
     For a \CL hypothesis class $\F_{\*\Theta}$, $T \in \mathbb{N}$ and any sequence of probability measures $\{\widehat{\mathbb{P}}_t\}_{t=1}^T$ such that $\widehat{\mathbb{P}}_t \in \mathcal{P}(\Y_t)^{\X_t} \subseteq \mathcal{P}(\Y)^{\X}$,
     %whose inputs and outputs are linked through the true conditional distribution $\{p^{\ast}(\*Y|x)\}_{x\in\X}$, 
     \CL algorithms are specified by functions
     \begin{IEEEeqnarray}{rCl}
        %\A_0: 2^{\*\Theta} & \to & \Info \times \*\Theta \nonumber \\
        \Ahat_{\*\I}: \*\Theta \times \Info \times \mathcal{P}(\Y)^{\X} & \to &  \Info \nonumber \\
        \Ahat_{\*\theta}: \*\Theta \times \Info \times \mathcal{P}(\Y)^{\X} & \to &  \*\Theta, \nonumber 
        \nonumber
     \end{IEEEeqnarray}
     where $\Info$ is some space that may vary between different \CL algorithms. 
     Given $\A_{\*\I}$ and $\A_{\*\theta}$ and some initializations $\*\theta_0$ and $\*\I_0$, \CL defines a procedure given by 
     \begin{IEEEeqnarray}{rCl}
        \*\theta_1 &=& \Ahat_{\*\theta}(\*\theta_0, \I_0, \widehat{\mathbb{P}}_1) \nonumber \\
        \I_1 & = & \Ahat_{\*\I}(\*\theta_1, \I_0, \widehat{\mathbb{P}}_1) \nonumber \\
        \*\theta_2 &=& \Ahat_{\*\theta}(\*\theta_1, \I_1, \widehat{\mathbb{P}}_2) \nonumber \\    
        \I_2 & = & \Ahat_{\*\I}(\*\theta_2, \I_1, \widehat{\mathbb{P}}_2) \nonumber \\
        &\dots& \nonumber \\
        \*\theta_T &=& \Ahat_{\*\theta}(\*\theta_{T-1}, \I_{T-1}, \widehat{\mathbb{P}}_T) \nonumber  \\
        \I_T & = & \Ahat_{\*\I}(\*\theta_{T}, \I_{T-1}, \widehat{\mathbb{P}}_T). \nonumber 
     \end{IEEEeqnarray}
     \label{def:CL}
\end{definition}

\begin{remark}
    In practice, the probability measures $\widehat{\mathbb{P}}_t$ will be empirical measures of $(\*X_t, \*Y_t)$ that are constructed from a finite number of samples.
\end{remark}
\begin{remark}
    An extremely attractive feature of the above definition is its generality.
    In partiular, the quantities $\I_t \in \Info$ are interpretable as \textit{any} kind of additional information carried forward through time. 
    While the role of these objects will differ between \CL algorithms, our definition is suitable to describe all of them.
    For example, in \textit{elastic weight consolidation} \citep{EWC, EWCLaplace}, $\I_t$ will be a diagonalized approximation of Fisher information.
    In contrast, \textit{variational continual learning} \citep{VariationalCL} will store the approximate posterior on the previous $t-1$ tasks as well as core sets for the previous tasks in $\I_t$.
    Generally, $\I_t$ dictates the memory requirements of any \CL algorithm.
    While the memory requirement is usually constant in the number of tasks, the \CL algorithm in \citet{ProgressiveNNs} would induce linearly growing memory  requirements, as it carries all previously fitted parameters $\{\*\theta_i\}_{i=1}^t$ forward in time.
    %
    %
    %\X{Make link with storage requirements: If $\Info$ potentially infinite-dimensional, one can just store a model per task. But obviously comes at storage cost [see parameter freezing stuff]}
    %
    %If $\I_t$ is stored in a constant number of bits, \CL has constant space complexity.
    %
    %Note that this excludes \CL algorithms whose space complexity grows linearly in time \X{cite parameter freezing stuff}.
\end{remark}
\begin{remark}
    Throughout the paper, whenever we write $\*\theta_t$, this value should be understood as a recursively defined function of all previously observed tasks $\{\widehat{\mathbb{P}}_i\}_{i=1}^t$:
    \begin{IEEEeqnarray}{rCl}
        \*\theta_t & = &
        \Ahat_{\*\theta}(\*\theta_{t-1}, \*I_{t-1}, \widehat{\mathbb{P}}_t)
        \nonumber \\
        & = & 
        \Ahat_{\*\theta}(
            \Ahat_{\*\theta}(\*\theta_{t-2}, \*I_{t-2}, \widehat{\mathbb{P}}_{t-1}), 
            \*I_{t-1}, 
            \widehat{\mathbb{P}}_t)
        \nonumber \\
        & = & 
        \Ahat_{\*\theta}(
            \Ahat_{\*\theta}(\*\theta_{t-2}, \*I_{t-2}, \widehat{\mathbb{P}}_{t-1}), 
            \Ahat_{\*\I}(\*\theta_{t-2}, \*I_{t-2}, \widehat{\mathbb{P}}_{t-1}), 
            \widehat{\mathbb{P}}_t)    
        \nonumber \\
        & = & \dots \nonumber
    \end{IEEEeqnarray}
    Clearly, a similar logic applies to the information $\I_t$ passed forward through time.
    Put differently, whenever we write $\*\theta_t$ and $\I_t$ throughout this paper, it is instructive to think about them as functions evaluated at all previous tasks, i.e.
    %specified by $\Ahat_{\*\theta}$ and $\Ahat_{\*I_t}$. Specifically,
    %
    \begin{IEEEeqnarray}{rCl}
        \*\theta_t & = &
        \mathcal{B}_{\*\theta}\left(\{\widehat{\mathbb{P}}_i\}_{i=1}^t\right) \nonumber \\
        \I_t & = &
        \mathcal{B}_{\I}\left(\{\widehat{\mathbb{P}}_i\}_{i=1}^t\right), \nonumber 
    \end{IEEEeqnarray}
    for functions $\mathcal{B}_{\*\theta}, \mathcal{B}_{\I}$ specified implicitly via $\Ahat_{\*\theta}$ and $\Ahat_{\*I_t}$.
    %\X{$\*\theta_t$ really is a function of the entire past, similarly for $\I_t$}  
\end{remark}

\begin{figure}[b!]
        %\vskip 0.1in
        \begin{center}
            \centerline{\includegraphics[trim= {0cm 0.0cm 0cm 0.0cm}, clip, %{1.5cm 3cm 2.2cm 3.6cm},clip, 
            width=1.00\columnwidth]{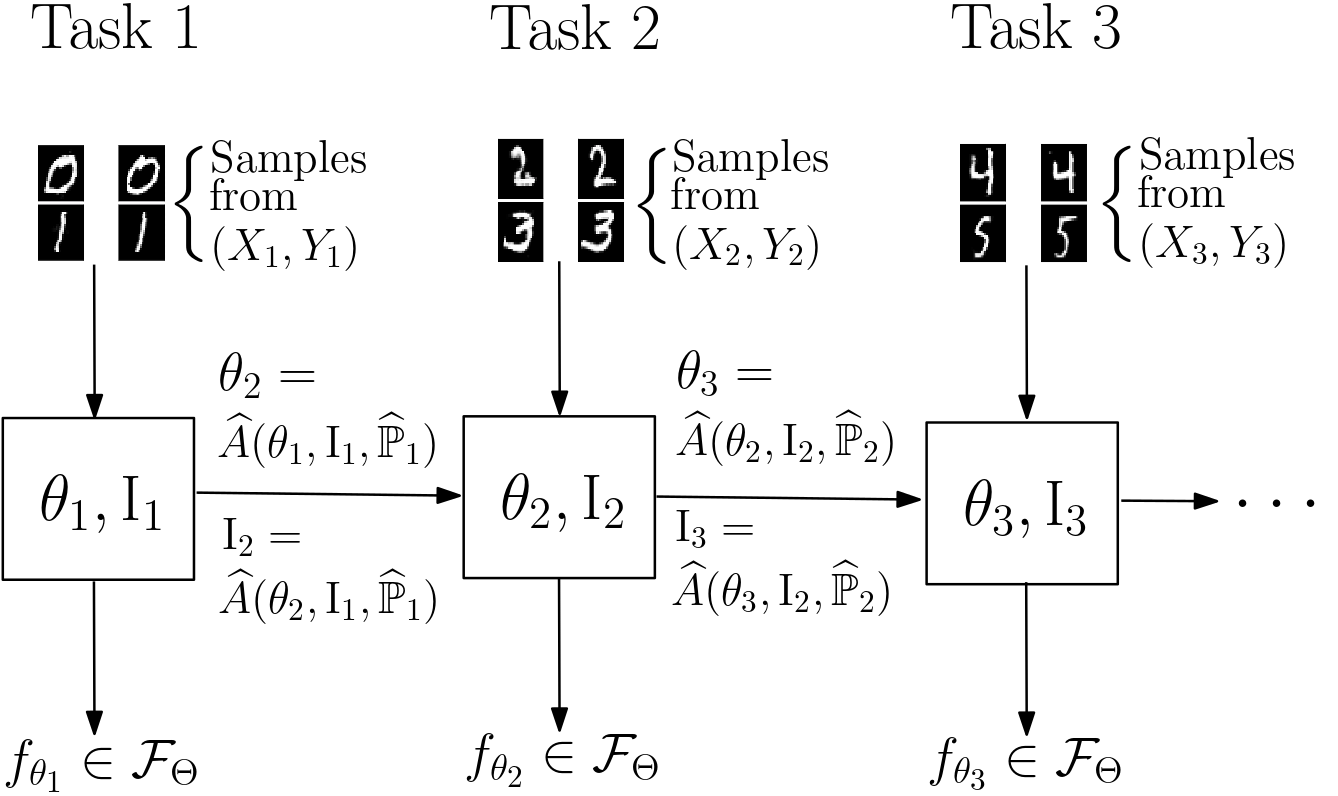}}
            \caption{
            Schematic of a generic \CL algorithm, using notation introduced in Section \ref{sec:prelim} and Definition \ref{def:CL}.
            }
            \label{Fig:CL_MNIST}
        \end{center}
        \vskip -0.2in
\end{figure}

While the literature on \CL has studied the problem of \textit{catastrophic forgetting} empirically, to the best of our knowledge no previous theoretical study has been attempted.
Thus, we first need to introduce a formal way of assessing whether a \CL algorithm suffers {catastrophic forgetting}.
As different researchers might disagree on the precise meaning of catastrophic forgetting, our formalism is very flexible.
In particular, all that it needs is an arbitrary binary-valued optimality criterion $\C$, whose function is to assess whether or not information of a task has been retained ($\C=1$) or forgotten ($\C=0$).
According to this formalism, a \CL algorithm avoids catastrophic forgetting (as judged by the criterion $\C$) if and only if its output at task $t$ is guaranteed to satisfy $\C$ on all previously seen tasks.
%
%Specifically, a given \CL algorithm's output relative to the criterion $\C$: If the \CL algorithm is optimal (with respect to $\C$), it will not have suffered catastrophic forgetting (as judged by $\C$).
%
%
%Next, we need to introduce a way to assess the performance of a given \CL algorithm relative to some optimality criterion $\C$.
%
%
%We do so by defining subsets of $\*\Theta$ that satisfy $\C$ for a given task.
%
%With this in hand, a \CL algorithm avoids catastrophic forgetting (as judged by the criterion $\C$) if and only if its output at task $t$ satisfies $\C$ on all previously seen tasks.
%
In this context, different ideas about the meaning of catastrophic forgetting would result in different choices for $\C$.
As we will analyze \CL with the tools of set theory, it is also convenient to define the function $\sat$, which maps from task distributions into the subsets consisting of all values in $\*\Theta$ which satisfy the criterion $\C$ on the given task.

\begin{definition}
     For an optimality criterion $\C: \*\Theta \times \mathcal{P}(\X \times \Y) \to \{0,1\}$
     and a set $\mathcal{Q} \subseteq \mathcal{P}(\X\times\Y)$ of task distributions,
     %and
     %a probability measure $\widehat{\mathbb{P}}$
     the function $\sat: \mathcal{P}(\X \times \Y) \to 2^{\*\Theta}$ defines the subset of $\*\Theta$ which satisfies $\C$ and
     %a sequence of probability measures $\{\widehat{\mathbb{P}}_t\}_{t=1}^T$, the set-valued function $\sat: \{1,2,\dots T\}  \to 2^{\*\Theta}$ 
     is given by
     \begin{IEEEeqnarray}{rCl}
         \sat(\widehat{\mathbb{P}}) & = &
         \{\*\theta \in \*\Theta: \C(\*\theta, \widehat{\mathbb{P}}) = 1\}.
         \nonumber
     \end{IEEEeqnarray}
     The collection of all possible sets generated by $\sat$ is 
     \begin{IEEEeqnarray}{rCl}
         \sat_{\mathcal{Q}} & = &
         \{
            \sat(\widehat{\mathbb{P}}): \widehat{\mathbb{P}} \in \mathcal{Q}
         \}
         \nonumber
     \end{IEEEeqnarray}
     and the collection of finite intersections from $\sat_{\mathcal{Q}}$ is
     \begin{IEEEeqnarray}{rCl}
         \sat_{\cap} & = & \{ \cap_{i=1}^t A_i: A_i \in \sat_{\mathcal{Q}}, \nonumber \\
         &&                 \quad 1 \leq i \leq t \text{ and } 1\leq t \leq T, \:\: T\in\mathbb{N}\}.
         \nonumber
     \end{IEEEeqnarray}  
     Lastly, for a given sequence $\{\widehat{\mathbb{P}}_t\}_{t=1}^T$ in $\mathcal{Q}$, define
     \begin{IEEEeqnarray}{rCl}
         \sat_t & = & \sat(\widehat{\mathbb{P}}_t) \nonumber \\
         \sat_{1:t} & = & \cap_{i=1}^t \sat_i, 
         \nonumber
     \end{IEEEeqnarray}
     for all $t=1,2,\dots T$.
\end{definition}
% \Y{I noticed that $\operatorname{SAT}(\mathcal{Q})$ can be confusing since it is initially defined as a mapping so perhaps use $\operatorname{SAT}_{\mathcal{Q}}$. I was mostly mentioning this since I would interpret $\operatorname{SAT}(\mathcal{Q})$ as 
% $$\operatorname{SAT}(\mathcal{Q}) = \bigcup_{q \in \mathcal{Q}} \operatorname{SAT}(q) $$
% }

\begin{definition}[Optimality]
     A \CL algorithm is optimal with respect to the criterion $\C$ and a set $\mathcal{Q}$ of task distributions if 
     \begin{itemize}
         \item[(i)] for any sequence $\{\widehat{\mathbb{P}}_t\}_{t=1}^T$ in $\mathcal{Q}$, 
     $\C(\*\theta_t, \widehat{\mathbb{P}}_i) = 1$, for all $i=1,2,\dots t$ and all $t=1,2,\dots T$;
        \item[(ii)] it holds that for any fixed $\*\theta', \*\I'$ that
        $\widehat{\A}_{\*\theta}(\*\theta', \*\I', \widehat{\mathbb{P}}) = \widehat{\A}_{\*\theta}(\*\theta', \*\I', \widehat{\mathbb{Q}})$ and
         $\widehat{\A}_{\*\I}(\*\theta', \*\I', \widehat{\mathbb{P}}) = \widehat{\A}_{\*\I}(\*\theta', \*\I', \widehat{\mathbb{Q}})$
         if $\sat(\widehat{\mathbb{Q}}) = \sat(\widehat{\mathbb{P}})$, for all $\widehat{\mathbb{P}}, \widehat{\mathbb{Q}}$ in $\mathcal{Q}$,
     \end{itemize}
     for any $T \in \mathbb{N}$.
     \label{def:optimality}
\end{definition}
\begin{remark}
    Note that we have defined optimality with respect to a possibly restricted subclass $\mathcal{Q} \subseteq \mathcal{P}(\X\times\Y)$ of task distributions.
    While the bulk of the literature on \CL makes no assumptions on the set of task distributions that the algorithm processes, this ensures that our notion of optimality could be made arbitrarily strict.
\end{remark}
\begin{remark}
    In spite of its name, the above definition only imposes weak restrictions on a \CL algorithm to be called optimal: All that it requires is that some arbitrary criterion $\C$ is satisfied on each task.
\end{remark}
\begin{remark}
    While this is notationally suppressed, the criterion $\C$ itself could depend on some hyperparameter.
    For example, suppose that for some loss function $\ell: \*\Theta \times \mathcal{X} \times \mathcal{Y} \to \mathbb{R}$ and some $\varepsilon \geq 0$, the optimality criterion is given by 
    \begin{IEEEeqnarray}{rCl}
        \C(\*\theta, \widehat{\mathbb{P}}) & = &
        \begin{cases}
            1 & \text{if } \int_{\X \times \Y}\ell(\*\theta, \*x, \*y) d\widehat{\mathbb{P}}(\*x, \*y) \leq \varepsilon \\
            0 & \text{otherwise}.
        \end{cases}
        \nonumber
    \end{IEEEeqnarray}
    Now $\C = \C_{\varepsilon}$ depends on $\varepsilon$ so that one can define optimality  for any \textit{fixed} value of $\varepsilon \geq 0$.
    %
    %
    %Clearly, the latter would be a stronger notion than the former.
    %
\end{remark}

\section{A Convenient Idealization}
\label{sec:idealization}

Throughout, it will be convenient to derive results relative to a version of \CL that is idealized. 
This idealization \textit{directly} has access to the sets $\sat_t$ (rather than to $\widehat{\mathbb{P}}_t$).
In other words, the idealization has access to a convenient oracle: It is already informed of all elements of the hypothesis class $\F_{\*\Theta}$ that satisfy the criterion $\C$ on the $t$-th task.
Importantly and as we will show in Lemma \ref{lemma:OCL=OICL},
studying optimality with the idealized version of \CL instead of the standard version does not impose any assumptions.
As the idealized version of \CL relies on basic set operations, this substantially simplifies the subsequent analysis.
%

%applies in the sense that rather than the empirical measure (i.e. sample) from the $t$-th task, the \CL algorithm is instead \textit{directly} granted access to the set $\sat_t$ containing all elements of the hypothesis class $\F_{\*\Theta}$ that satisfy the criterion $\C$ on the $t$-th task.

%Notice that this constitutes the best possible learning outcome one could wish to have on the $t$-th task.
%
%This strengthens the results derived in this paper: The results hold \textit{even if} one were sequentially granted access to a task-specific oracle.
%
\begin{definition}[Idealized Continual Learning]
     For a hypothesis class $\F_{\*\Theta}$, any $T \in \mathbb{N}$ and any sequence of sets $\{\sat_t\}_{t=1}^T$ 
     generated with some fixed criterion $\C$ and an arbitrary sequence of probability measures $\{\widehat{\mathbb{P}}_t\}_{t=1}^T$ as in Definition \ref{def:optimality},
     %whose inputs and outputs are linked through the true conditional distribution $\{p^{\ast}(\*Y|x)\}_{x\in\X}$, 
     Idealized \CL (\ICL) algorithms are specified by functions $\A_{\I}$ and $\A_{\*\theta}$ 
     \begin{IEEEeqnarray}{rCl}
        %\A_0: 2^{\*\Theta} & \to & \Info \times \*\Theta \nonumber \\
        \A_{\*\I}: \*\Theta \times \Info \times \sat_{\mathcal{Q}} & \to &  \Info \nonumber \\
        \A_{\*\theta}: \*\Theta \times \Info \times \sat_{\mathcal{Q}} & \to &  \*\Theta, \nonumber 
        \nonumber
     \end{IEEEeqnarray}
     where $\Info$ is some space that may vary between different \ICL algorithms. 
     Given $\A_{\*\I}$ and $\A_{\*\theta}$ and some initializations $\*\theta_0$ and $\*\I_0$, \ICL defines a procedure given by 
     \begin{IEEEeqnarray}{rCl}
        \*\theta_1 &=& \A_{\*\theta}(\*\theta_0, \I_0, \sat_1) \nonumber \\
        \I_1 & = & \A_{\*\I}(\*\theta_1, \I_0, \sat_1) \nonumber \\
        \*\theta_2 &=& \A_{\*\theta}(\*\theta_1, \I_1, \sat_2) \nonumber \\    
        \I_2 & = & \A_{\*\I}(\*\theta_2, \I_1, \sat_2) \nonumber \\
        &\dots& \nonumber \\
        \*\theta_T &=& \A_{\*\theta}(\*\theta_{T-1}, \I_{T-1}, \sat_T) \nonumber  \\
        \I_T & = & \A_{\*\I}(\*\theta_{T}, \I_{T-1}, \sat_T). \nonumber 
     \end{IEEEeqnarray}
     \label{def:ICL}
\end{definition}
%

%\begin{remark}
    It should be clear that the only difference between \CL and Idealized \CL is the third argument: Rather than using $\widehat{\mathbb{P}}_t$, Idealized \CL algorithms use $\sat_t$.
    Apart from that, everything else remains the same: $\I_t$ is still interpretable as additional information and  $\widehat{\mathbb{P}}_t$ as an empirical measure composed of samples from the $t$-th task.
%\end{remark}

%
\begin{definition}[Optimal Idealized Continual Learning]
     An \ICL algorithm is an \OICL procedure if $\*\theta_t \in \sat_{1:t}$ for all $t=1,2,\dots T$.
    \label{def:OICL}
\end{definition}

\begin{figure}[t!]
        %\vskip 0.1in
        \begin{center}
            \centerline{\includegraphics[trim= {0cm 0.0cm 0cm 0.0cm}, clip, %{1.5cm 3cm 2.2cm 3.6cm},clip, 
            width=1.00\columnwidth]{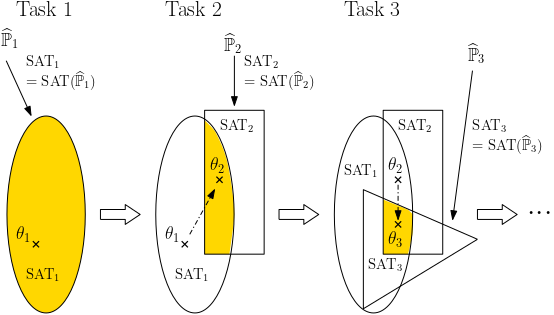}}
            \caption{
            Visualization of optimal Idealized CL: At each task, the new value for $\*\theta_t$ must lie in the intersection $\sat_{1:t} = \cap_{i=1}^t\sat_i$.
            }
            \label{Fig:CL_set_perspective}
        \end{center}
        \vskip -0.2in
\end{figure}

In contrast to an optimal \CL algorithm, an \OICL algorithm has a clear set-theoretic interpretation that is set out in Figure \ref{Fig:CL_set_perspective}: $\*\theta_t$ needs to lie in the intersection of the sets that mark out the subspaces of $\*\Theta$ on which $\C$ is satisfied relative to all $t$ tasks observed thus far. 
Combined with the next result, this will serve to abstract and simplify the further analysis of optimal \CL.

% \begin{remark}
%     Optimal Idealized \CL enshrines what one may intuitively expect an \ICL method to behave like: 
%     %Since our hypothesis class $\F$ is large enough to contain the true conditional, we want $\*\theta_t$ to be \textit{optimal} on the first $t$ tasks processed thus far.
%     When sequentially presented with task-specific oracles specified via $\{\sat_t\}_{t=1}^T$, the algorithm should be able to learn an element that satisfies the criterion on all tasks.
%     %
%     By definition of the sets $\sat_t$, presenting the algorithm with this oracle at each task amounts to presenting the algorithm with the full joint distribution of $\mathbb{P}_t$ of $(\*X_t, \*Y_t)$ at each task $t$.
% \end{remark}
%

%\X{Show set-theoretic implication of this \CL algo}

%
\begin{lemma}
    Any optimal \CL algorithm is an optimal \ICL algorithm relative to the same criterion.
    \label{lemma:OCL=OICL}
\end{lemma}
\begin{proof}

     Suppose that $\widehat{\A}_{\*\theta}$ and $\widehat{\A}_{\*\I}$ define an optimal \CL algorithm.
     Simply define ${\A}_{\*\theta}(\*\theta_t, \*\I_t, \sat(\widehat{\mathbb{P}}_t)) = \widehat{\A}_{\*\theta}(\*\theta_t, \*\I_t, \widehat{\mathbb{P}}_{t})$ and similarly ${\A}_{\*\I}(\*\theta_t, \*\I_t, \sat(\widehat{\mathbb{P}}_t)) = \widehat{\A}_{\*\I}(\*\theta_t, \*\I_t, \widehat{\mathbb{P}}_t)$ as the functions specifying the corresponding \ICL algorithm.
     %makes this obvious.
     %
     By definition, $\sat_t = \sat(\widehat{\mathbb{P}_t})$ so that the reverse is immediate, too.
\end{proof}
%
%\begin{remark}
    %
    %The above elucidates why  the concept of Idealized \CL as given in Definition \ref{def:ICL} is of interest: Any optimal \CL algorithm is in fact equivalent to an Optimal Idealized \CL algorithm.
%
    Lemma \ref{lemma:OCL=OICL} plays a central role throughout the rest of the paper: 
    Based on the stated equivalence, one can use idealized optimal \CL algorithms to analyze standard optimal \CL algorithms.
    This has two advantages: %
    Firstly, it drastically simplifies the analysis by reducing it to basic set theory.
    Secondly, it provides new ways of forming intuitions about the computational properties of optimal \CL algorithms.
    %
%\end{remark}

%

%
%In other words, performance guarantees are generally impossible unless very careful assumptions are made on both $\{\widehat{\mathbb{P}}_t\}_{t=1}^T$ and $\F_{\*\Theta}$.
%
%optimal \CL is generally impossible unless $\I_t$ can store arbitrarily much information or if $\{\sat_t\}_{t=1}^T$ has very specific geometric structure.

%
%In particular, we show that optimal \CL is not achievable unless $\I_t$ and $\*\theta_t$ suffice to reconstruct $\O^{\*\Theta}_{1:t}$ completely. This is possible only if $\I_t = \O^{\*\Theta}_{1:t}$ or if the sets $\{\O_t^{\*\Theta}\}_{t=1}^T$ have a particular structure, e.g. if $\O_t^{\*\Theta} = \O^{\*\Theta}_{1:t}$ for all $t$.
%

\section{Main Results}
\label{section:main_result}

Next, we summarize the main results: Generally,
\begin{itemize}
    \item[(1)] 
    optimal \CL algorithms are \NPH and
    \item[(2)]
    optimal \CL algorithms require perfect memory.
\end{itemize}
While we sketch the most important proofs, full details and derivations are deferred to the Appendix.
%The full details are contained in the appendix, but we convey important key ideas and intuitions by giving high-level proof sketches and using visual aids. 
%
To clearly convey the most important insights, we additionally provide examples and illustrations.

Before proceeding, we state another key lemma that is invaluable for both main results.
Its role is to lower bound both the memory requirement and computational hardness of optimal \CL algorithms with that of a well-studied decision problem which is illustrated in Figure \ref{Fig:decision_problem}.

\begin{lemma}
    An optimal \CL algorithm is computationally at least as hard as deciding whether $A \cap B = \emptyset$, for $A \in \sat_{\cap}$ and $B \in \sat_{\mathcal{Q}}$.
    \label{lemma:decision_problem}
\end{lemma}
\begin{proofsketch}
    By virtue of Lemma \ref{lemma:OCL=OICL}, it suffices to show this for the corresponding optimal \ICL algorithm.
    Since $\*\theta_t \in \sat_{1:t}$, optimal \ICL solves a particular optimization problem: In particular, optimal \ICL finds a $\*\theta_t \in A \cap B$, for some $A = \sat_{1:(t-1)} \in \sat_{\cap}$ and $B = \sat_{t} \in \sat_{\mathcal{Q}}$.
    Clearly, finding an element in $A\cap B$ is at least as hard as determining whether $A\cap B = \emptyset$. 
\end{proofsketch}

\begin{figure}[h!]
        %\vskip 0.1in
        \begin{center}
            \centerline{\includegraphics[trim= {0cm 0.0cm 0cm 0.0cm}, clip, %{1.5cm 3cm 2.2cm 3.6cm},clip, 
            width=1.00\columnwidth]{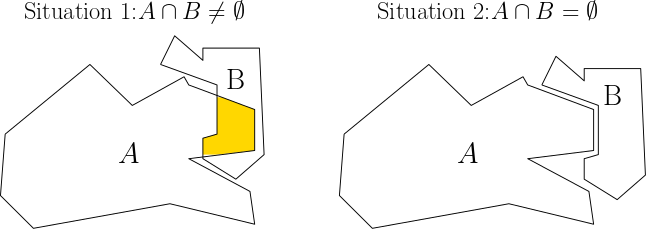}}
            \caption{
            Any \CL algorithm induces the two collections of sets $\sat_{\mathcal{Q}}$ and $\sat_{\cap}$.
            Lemma \ref{lemma:decision_problem} says that if the \CL algorithm is optimal, then it solves a problem at least as hard as deciding if $A \cap B = \emptyset$, for all $A \in \sat_{\cap}$ and $B \in \sat_{\mathcal{Q}}$.
            }
            \label{Fig:decision_problem}
        \end{center}
        \vskip -0.2in
\end{figure}

\subsection{Computational Complexity}
\label{sec:NP}

Finally, we are in a position to formally state our  result on the computational hardness of optimal \CL.

\begin{theorem}
    If $\mathcal{Q}$ and $\C$ are such that $\sat_{\mathcal{Q}} \supseteq S$ or $\sat_{\cap} \supseteq S$
    so that $S$ is the set of tropical hypersurfaces  or the set of polytopes on $\*\Theta$, 
    then optimal \CL is \NPH.
    \label{thm:NP}
\end{theorem}
\begin{proofsketch}
    First, we use Lemma \ref{lemma:decision_problem}: Optimal \CL can correctly decide if $A\cap B = \emptyset$, for all $A \in \sat_{\cap}$ and $B \in \sat_{\mathcal{Q}}$. 
    %This means that it can be used to decide for any $A \in \sat_{\cap}$ and $B \in \sat_{\mathcal{Q}}$ whether  $A\cap B = \emptyset$.
    %
    %By standard arguments, this can be used to show that this is possible only if they 
    %
    Second, we use established reductions to conclude that this decision problem is \NPC. 
    For the case where $S$ is the set of tropical hypersurfaces, the results in \citet{AlgebraicVarieties} can be used.
    If $S$ is the set of polytopes, the same conclusion is reached by using the results of \citet{polyhedralPaper} and \citet{polyhedralThesis}.
    Third, it then follows by standard arguments that the optimization problem corresponding to an \NPC decision problem is \NPH. 
    %to conclude that deciding whether the intersection $A \cap B$ is empty will be in $\NP$.
\end{proofsketch}

    One may wonder how consequential the above result is in practice.
    Specifically, which kind of criterion $\C$ and which kind of model would produce polytopes or tropical hypersurfaces?
    In fact, relatively simple models and optimality criteria suffice to produce such adverse solution sets. 
    We showcase this in the next example: As we shall see, a simple linear model together with an intuitively appealing upper bound on the prediction error as optimality criterion are sufficient to make the corresponding optimal \CL problem \NPH.

\begin{example}
    Take $\F_{\*\Theta}$ to be the collection of linear models with inputs on $\X$ and outputs on $\Y \subset \mathbb{R}$ linked through the coefficient vector $\*\theta \in \*\Theta$.
    Further, let $\mathcal{Q}$ be the collection of  empirical measures 
    \begin{IEEEeqnarray}{rCl}
        \widehat{m}^t(y,x) & = & \frac{1}{n_t}\sum_{i=1}^{n_t}\delta_{(y_i^t,x_i^t)}(y,x)
        \nonumber
    \end{IEEEeqnarray}
    whose $n_t \in \mathbb{N}$ atoms $\{(y_i^t, x_i^t)\}_{i=1}^{n_t}$ represent the $t$-th task.
    %
    % are sampled as
    % \begin{IEEEeqnarray}{rCl}
    %     x_i^t & \overset{iid}{\sim} &
    %         \mathcal{N}\left(\*x; \*\mu^t, \*\Sigma_x^t\right)
    %         \nonumber \\
    %     y_i^t|x_i^t & \sim &
    %         \mathcal{N}\left(\*y; (x_i^t)^T \*\theta^t, \*\sigma_y^t \right)
    %         \nonumber
    % \end{IEEEeqnarray}
    % %
    % where $t=1,2,\dots $ indexes potentially different values for the vectors $\*\mu^t, \*\theta^t \in \mathbb{R}^d$, the covariance matrix $\*\Sigma_x^t \in \mathbb{R}^{d\times d}$ and the scalar $\*\sigma_y^t > 0$.
    % %
    % In words, consider the collection of probability distributions defined through traditional regression with normal errors and regressors.
    %
    Further, define for $\varepsilon \geq 0$ and all $\widehat{\mathbb{P}} \in \mathcal{Q}$ the criterion
    \begin{IEEEeqnarray}{rCl}
        \C(\*\theta, \widehat{\mathbb{P}}) & = &
        \begin{cases}
            1 & \text{if } |y_i^t - \*\theta^Tx_i^t| \leq \varepsilon, \text { for all } i=1,2,\dots n_t \\
            0 & \text{otherwise}.
        \end{cases}
        \nonumber
    \end{IEEEeqnarray}
    %
    %Notice that this optimality criterion would be trivial 
    %
    Then, it is straightforward to see that
    \begin{IEEEeqnarray}{rCl}
        \sat(\widehat{\mathbb{P}}) & = &
        \big\{
            \*\theta \in \*\Theta: y_i^t - \*\theta^Tx_i^t \leq \varepsilon \text{ and }
             y_i^t - \*\theta^Tx_i^t -\geq
            \varepsilon, \nonumber \\
        && 
            \text{ for all } i=1,2,\dots n_t
        \big\} \nonumber\\
        & = &
        \left(
        \cap_{i=1}^{n_t}\big\{ 
            \*\theta \in \*\Theta: y_i^t - \*\theta^Tx_i^t \leq \varepsilon \big\}
        \right)
        \cap \nonumber \\
       && \left(\cap_{i=1}^{n_t}\big\{ 
            \*\theta \in \*\Theta: y_i^t - \*\theta^Tx_i^t \geq -\varepsilon \big\}
        \right),
        \nonumber
    \end{IEEEeqnarray}
    which is an intersection of $2n_t$ half-spaces in $\mathbb{R}^d$ and thus a polytope.
    %
    %As long as $n_t \in \mathbb{N}$, $\mu^t$, $\*\Sigma^t_x$, $\*\theta^t$ and $\*\sigma_y^t$ are arbitrary, 
    Unless we make very strong assumptions about the common structure between the task distributions which generated the atoms, this implies that we could recover \textit{any} given  polytope in $\mathbb{R}^d$ by constructing $\sat(\widehat{\mathbb{P}})$.
    Under these circumstances, the conditions of Theorem \ref{thm:NP} are met: $\sat_{\mathcal{Q}}$ contains all polytopes.
    While the criterion may seem strict, optimal \CL would still suffer the same problem even if we used the alternative criterion
    \begin{IEEEeqnarray}{rCl}
        \C(\*\theta, \widehat{\mathbb{P}}) & = &
        \begin{cases}
            1 & \text{if } \frac{1}{n}\sum_{i=1}^n|y_i^t - \*\theta^Tx_i^t| \leq \varepsilon \\
            0 & \text{otherwise}.
        \end{cases}
        \nonumber
    \end{IEEEeqnarray}
    In this case, $\sat(\widehat{\mathbb{P}})$ would still be a polytope in $\mathbb{R}^d$, albeit made up only of $2d$ intersections of half-spaces.
    By similar reasoning as applied to $\sat_{\mathcal{Q}}$ before, the collection $\sat_{\cap}$ now contains any arbitrary polytope in $\mathbb{R}^d$.
    \begin{figure}[t!]
        %\vskip 0.1in
        \begin{center}
            \centerline{\includegraphics[trim= {0.cm 2.0cm 0.cm 0.cm}, clip, %{1.5cm 3cm 2.2cm 3.6cm},clip, 
            width=1.00\columnwidth]{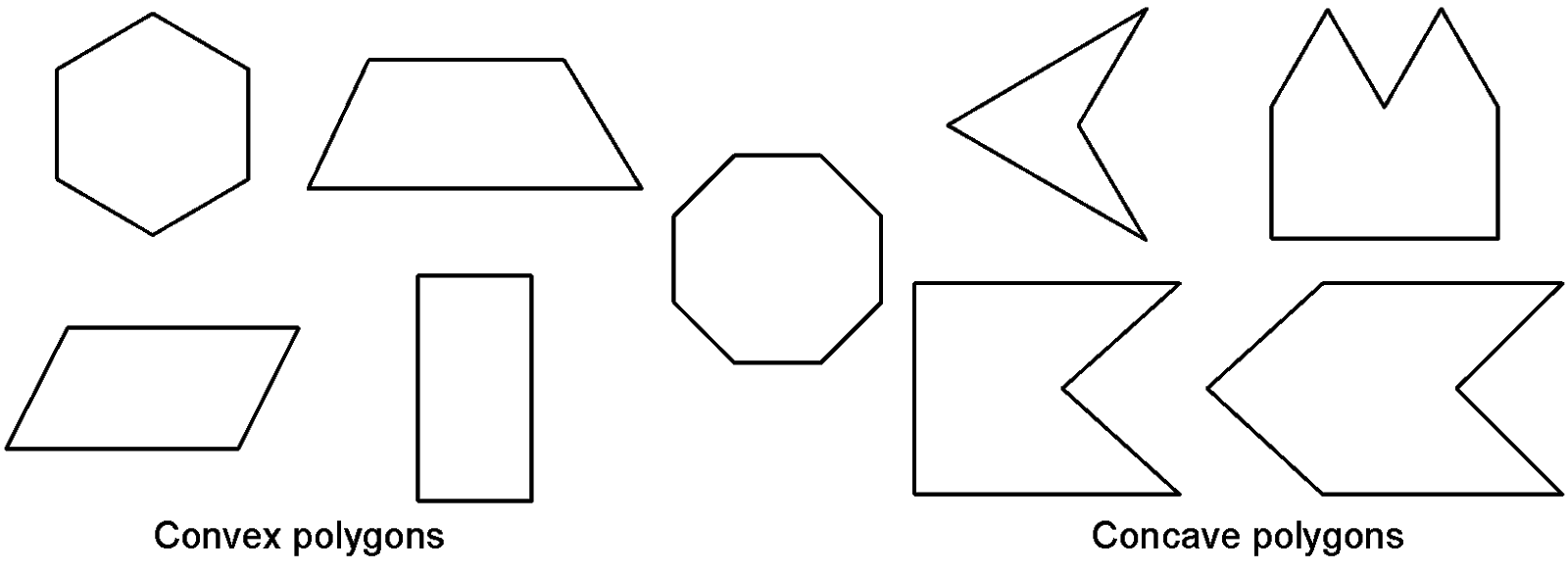}}
            \caption{
            A small selection of polytopes in $\mathbb{R}^2$. Polytopes are basic and well-studied studied geometric shapes.
            }
            \label{Fig:polytopes}
        \end{center}
        \vskip -0.2in
    \end{figure}
    %
    %In practice, both the \CL hypothesis class as well as the data distributions should be expected to be far more complicated so that 
    \label{example:LAD_regression_NPH}
\end{example}

Summarizing Example \ref{example:LAD_regression_NPH}, the optimal \CL problem is \NPH even for extremely simple models and even when we restrict the set of permissible data distributions to be almost comically simple.
Indeed, the shapes depicted in Figure \ref{Fig:polytopes} are clearly far simpler than the sets of optimal solutions $\sat_t$ that would be induced by non-linear hypothesis classes $\F_{\*\Theta}$ such as Artificial Neural Networks.
%
%\textcolor{blue}{Add something in here about NNs!}
%
%for \CL hypothesis classes $\F_{\*\Theta}$  of practical interest.
%
In other words, real world \CL algorithms based on Deep Learning will induce a collection of sets $\sat_{\mathcal{Q}}$ whose elements are at least as hard to intersect as polytopes.
Since even the intersection of the geometrically relatively simple polytopes is \NPH, it is straightforward to show that such \CL algorithms solve an \NPH problem.
The next Corollary formalizes this observation.

\begin{corollary}
    Suppose that $\sat_{\mathcal{Q}} \supseteq S$, with $S$ being a collection of sets for which deciding if $A\cap B = \emptyset$ for $A, B \in S$ is computationally at least as hard as for the collection of polytopes in $\*\Theta$.
    Then optimal \CL is \NPH.
    \label{corollary:NPH}
\end{corollary}
%
% \begin{proof}
%     This follows by appropriately modifying the second and third part in the proof of Theorem \ref{thm:NP}.
% \end{proof}
%

\subsection{Memory Requirements}
\label{sec:memory}

Having established the computational hardness of optimal \CL, we next investigate its memory requirements.
To this end, we first need to develop a notion of perfect memory.
Specifically, we will define perfect memory to be the most memory-efficient way of retaining all solutions that have not been ruled out by previously processed tasks.
The first step along this path is the definition of equivalence sets.
Intuitively speaking, equivalence sets are subsets of $\*\Theta$ whose values perform exactly the same across all tasks in $\mathcal{Q}$ (as judged by the criterion $\C$).

\begin{definition}[Equivalence set]
     For $\*\theta \in \*\Theta$, define $S(\*\theta) = \{A \in \sat_{\mathcal{Q}}: 
            \*\theta \in A \}$ and the equivalence sets
     \begin{IEEEeqnarray}{rCl}
         \E(\*\theta)
         & = &
         \bigcap_{
                A \in S(\*\theta)
        }A.
        \nonumber
     \end{IEEEeqnarray}
\end{definition}

\begin{remark}
    Equivalence sets are constructed as illustrated in Figure \ref{Fig:equivalence_sets}. 
    Thus, any set $\E(\*\theta)$ contains all  other solutions $\*\theta' \in \*\Theta$ which are as good as $\*\theta$. 
    To illustrate this logic, suppose $\E(\*\theta) = \{\*\theta\}$.
    In this case, for each value of $\*\theta \in \*\Theta$ there is \textbf{no} other value $\*\theta' \in \*\Theta$ that is guaranteed to perform equally well as $\*\theta$ across all tasks in $\mathcal{Q}$.
    %, regardless of the sequence of distributions $\{\widehat{\mathbb{P}}_t\}_{t=1}^T$.
    %
    %Conversely, suppose that $\E(\*\theta) = A$ for some non-singleton set $A \subset \*\Theta$. Then \textit{any} value in $A$ is guaranteed to perform equally well as $\*\theta$ across all tasks in $\mathcal{Q}$.
    %regardless of the sequence of distributions $\{\widehat{\mathbb{P}}_t\}_{t=1}^T$.
\end{remark}

\begin{figure}[b!]
        %\vskip 0.1in
        \begin{center}
            \centerline{\includegraphics[trim= {0cm 0.0cm 0cm 0.0cm}, clip, %{1.5cm 3cm 2.2cm 3.6cm},clip, 
            width=1.00\columnwidth]{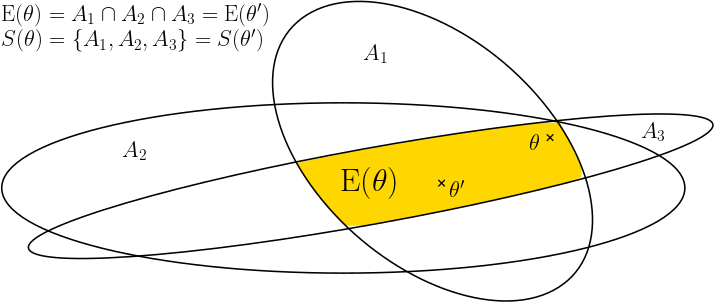}}
            \caption{
            Illustration of Equivalence sets and the results of Lemma \ref{lemma:equivalence_set_properties}: $\E(\*\theta)$ is the intersection of $S(\*\theta)$, the collection of all sets in $\sat_{\mathcal{Q}}$ that contain $\*\theta$.
            }
         \label{Fig:equivalence_sets}
        \end{center}
        \vskip -0.2in
\end{figure}

Equivalence sets satisfy a number of important properties summarized in the following Lemma.

\begin{lemma}
    For arbitrary equivalence sets $\E(\*\theta), \E(\*\theta')$ and arbitrary $A\in\sat_{\mathcal{Q}}$, it holds that
    \begin{itemize}
        \item $\*\theta' \in \E(\*\theta) \Longleftrightarrow \E(\*\theta) = \E(\*\theta')$;
        \item If $\*\theta' \notin \E(\*\theta)$, then $\E(\*\theta) \cap \E(\*\theta') = \emptyset$;
        \item Either $\E(\*\theta) \subseteq A$ or $\E(\*\theta) \cap A = \emptyset$.
    \end{itemize}
    %
    %$\*\theta' \in \E(\*\theta) \Longleftrightarrow \E(\*\theta) = \E(\*\theta')$. Further, whenever $\*\theta' \notin \E(\*\theta)$, it hold that $\E(\*\theta) \cap \E(\*\theta') = \emptyset$. 
    %
    %Moreover, for all $A \in \sat_{\mathcal{Q}}$ and all $\*\theta \in \*\Theta$, it holds that either $\E(\*\theta) \subseteq A$ or $\E(\*\theta) \cap A = \emptyset$.
    %
    \label{lemma:equivalence_set_properties}
\end{lemma}

Next, we formally define perfect memory in the context of \CL algorithms.
In particular, we will say that a \CL algorithm has perfect memory if it can
reconstruct at least one element of each equivalence set whose elements satisfy the optimality criterion $\C$ on all tasks observed thus far.
To this end, we define Minimal Covers and Minimal Representations.
Figure \ref{Fig:minimal_cover_and_representation} illustrates both concepts.

\begin{definition}[Minimal cover]
    Given an indexed set $\{\*\theta_i\}_{i\in I}$ of points in $\*\Theta$, suppose that $\{\E(\*\theta_i)\}_{i\in I}$ forms a cover of $\cup_{A \in \sat_{\mathcal{Q}}}A$ such that for $i\neq j$, $\E(\*\theta_i)\cap\E(\*\theta_j)=\emptyset$. 
    Then we call such a cover minimal. 
\end{definition}

\begin{remark}
    Note that by virtue of Lemma \ref{lemma:equivalence_set_properties}, any Minimal cover forms a non-overlapping and unique partition of $\cup_{A \in \sat_{\mathcal{Q}}}A$.
    Figure \ref{Fig:minimal_cover_and_representation} illustrates this point.
\end{remark}

\begin{definition}[Minimal representation]
    Let $\{\E(\*\theta_i)\}_{i\in I}$ be a minimal cover.
    Denoting by $f:I\to\*\Theta$ a function such that $f(i) \in \E(\*\theta_i)$ for all $i\in I$, we call the set $\{f(i)\}_{i\in I}$ a minimal representation of $\cup_{A \in \sat_{\mathcal{Q}}}A$.
    %
    %For convenience, we also write $F = \cup_{i\in I}f(i)$.
\end{definition}
%
%\begin{remark}
    In the context of a \CL algorithm, the minimal representation is the smallest possible set in $\*\Theta$ that one needs to to retain all potential solutions of different quality (as judged by the criterion $\C$).
    In fact, it is instructive to think of a minimal representation as the most memory-efficient representation of the set of all potential solutions: Since all points in an equivalence set are equally good under $\C$ by definition, one can store a single point for each equivalence class $\E(\*\theta) \subset \*\Theta$ without losing information. 
    In other words, a minimal representation is the most memory-efficient way of retaining perfect memory.
    %
%\end{remark}

\begin{figure}[b!]
        %\vskip 0.1in
        \begin{center}
            \centerline{\includegraphics[trim= {0cm 0.0cm 0cm 0.0cm}, clip, %{1.5cm 3cm 2.2cm 3.6cm},clip, 
            width=1.00\columnwidth]{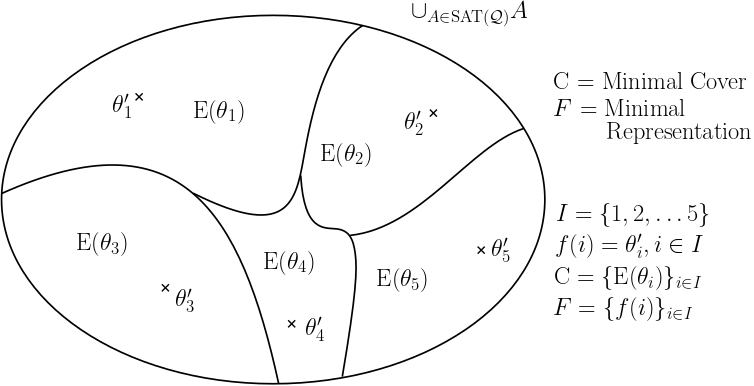}}
            \caption{
            Illustration of Minimal Covers and Representations. Unlike Minimal Covers, Minimal Representations are \textit{not} unique: Both $f(i) = \*\theta_i$ and $f(i) = \*\theta_i'$ yield Minimal Representations.
            }
            \label{Fig:minimal_cover_and_representation}
        \end{center}
        \vskip -0.3in
\end{figure}

%\X{Forgot to add that $C_t \subset \sat_{1:t}$ in the definition! Otherwise, can always have $h = \*\Theta$...}

\begin{definition}[Perfect memory]
     We say that an optimal \CL algorithm has perfect memory if there exists a function $h:\*\Theta \times \Info \to 2^{\*\Theta}$ for which $h(\*\theta_t, \I_t) = C_t$ such that $\sat_{1:t} \supseteq C_t \supseteq \left(\cup_{i\in I}f(i)\right)\cap\sat_{1:t}$ at task $(t+1)$, for some fixed but arbitrary minimal representation $\{f(i)\}_{i\in I}$ and for any arbitrary $\sat_{1:t} \in \sat_{\cap}$.
     \label{definition:perfect_memory}
\end{definition}
\begin{remark}
    The above conceptualizes an intuitive notion of perfect memory: The set $\left(\cup_{i\in I}f(i)\right)\cap\sat_{1:t}$ contains exactly one value for each equivalence set whose solutions are still optimal given the first $t$ tasks. 
    Thus, an optimal \CL algorithm has perfect memory if it can reconstruct at least one value for each equivalence set that has not been ruled out as sub-optimal by the $t$ preceding tasks.
    Equivalently, one could say that an optimal \CL algorithm can perfectly memorize all equivalence sets ruled out by the first $t$ tasks.
\end{remark}

%\X{Picture showing what perfect memory means (use the CL MNIST PIC but with equivalence sets/minimal repr.)}

We are now almost in a position to show that optimal \CL algorithms will generally have perfect memory.
The last missing ingredient is the following lemma.

%\X{Give proof elsewhere}
\begin{lemma}
    If a \CL algorithm is optimal, there exists $h: \*\Theta \times \Info \to 2^{\*\Theta}$ for which $h(\*\theta_t, \I_t) = C_t$ is such that
    $C_t \cap A = \emptyset \Longleftrightarrow \sat_{1:t} \cap A = \emptyset$, for all $A \in \sat_{\mathcal{Q}}$.
    \label{lemma:memory_lemma}
\end{lemma}

With this, all that is left to do is proving that $C_t$ of Lemma \ref{lemma:memory_lemma} is contained by $\sat_{1:t}$ and contains  $\left(\cup_{i\in I}f(i)\right)\cap\sat_{1:t}$, for some Minimal Representation $\{f(i)\}_{i\in I}$.
Under mild regularity conditions, this yields the second main result.

\begin{theorem}
    Suppose that for an optimal \CL algorithm, $C_t \subseteq \sat_{1:t}$, $C_t \subseteq C_{t-1}$ and that for all $\*\theta \in \*\Theta$ there exists $\{A_t\}_{t=1}^T$ in $\sat_{\mathcal{Q}}$ such that $\cap_{t=1}^TA_t = \E(\*\theta)$. Then this optimal \CL algorithm has perfect memory.
    \label{thm:optimal_CL_perfect_memory}
\end{theorem}
\begin{proofsketch}
    By Lemma \ref{lemma:memory_lemma}, we know that $C_t$ will suffice to solve the decision problem already discussed in the proof sketch of Theorem \ref{thm:NP}.
    We also know that any optimal \CL algorithm has to be able to solve this decision problem. 
    Thus, the memory requirements of the decision problem lower bound those of optimal \CL algorithms.
    Next, we prove that there must exist a minimal representation $\{f(i)\}_{i\in I}$ for which $C_t \supseteq (\cup_{i\in I}\{f(i)\}) \cap \sat_{1:t}$, for all $t = 1,2,\dots T$.
    We do this by combining Lemma \ref{lemma:equivalence_set_properties} with the additional conditions imposed upon $C_t$.
\end{proofsketch}

The conditions imposed in Theorem \ref{thm:optimal_CL_perfect_memory} are general, but also rather abstract.
Alternatively, much stronger conditions with a more straightforward interpretation could be imposed to derive the same result.
%
%The following Corollary derives the same results from much simpler conditions.

%
\begin{corollary}
    Suppose $\C$ and $\mathcal{Q}$ are such that $\E(\*\theta) \in \sat_{\mathcal{Q}}$ for all $\*\theta \in \*\Theta$. Then any optimal \CL algorithm has perfect memory.
    \label{corollary:optimal_CL_perfect_memory}
\end{corollary}
Even though the conditions of Corollary \ref{corollary:optimal_CL_perfect_memory} are far more restrictive than those of Theorem \ref{thm:optimal_CL_perfect_memory},  it is relatively easy to find  examples on which they hold.
\begin{example}
    To keep things simple, consider again the set-up of Example \ref{example:LAD_regression_NPH}.
    %
    %Further, consider a \CL algorithm that is optimal regardless of the choice for $\varepsilon\geq 0$.
    %
    Consider
    \begin{IEEEeqnarray}{rCl}
        \varepsilon_t  & = &
        \min_{\*\theta}\frac{1}{n}\sum_{i=1}^n|y_i^t - \*\theta^T x_i^t|,
        \nonumber \\
        \*\theta^{\ast}_t & = &
        \arg\min_{\*\theta}\frac{1}{n}\sum_{i=1}^n|y_i^t - \*\theta^T x_i^t|.
        \nonumber
    \end{IEEEeqnarray}
    %
    %is a valid choice. 
    %
    Unless we substantially restrict the permitted task distributions in $\mathcal{Q}$, we cannot exclude the possibility that $\E(\*\theta) = \{\*\theta\} \in \sat_{\mathcal{Q}}$ for all $\*\theta \in \*\Theta$.
    Specifically, this is the case if for a fixed $\varepsilon \geq 0$ and for {any} $\*\theta \in \cup_{A \in \sat_{\mathcal{Q}}}A$, it is possible to find an empirical measure in $\mathcal{Q}$ constructed with atoms $\{(y_i^t, x_i^t)\}_{i=1}^{n_t}$ for which $\*\theta = \*\theta^{\ast}_t$ so that
    \begin{IEEEeqnarray}{rCCCCCCCl}
        \frac{1}{n}\sum_{i=1}^n|y_i^t - \*\theta^T x_i^t|
        & = &
        \varepsilon & = &
        \varepsilon_t & = &
        \frac{1}{n}\sum_{i=1}^n|y_i^t - (\*\theta_t^{\ast})^T x_i^t|.
        \nonumber
    \end{IEEEeqnarray}
    For the corresponding empirical distribution $\widehat{\mathbb{P}}_t$, it would then follow that $\sat(\widehat{\mathbb{P}}_t) = \{\*\theta\}$.
    Since $\*\theta$ was chosen arbitrarily, this immediately entails that
    \begin{IEEEeqnarray}{rCl}
        \E(\*\theta) & = & \{\*\theta\} \in  \sat_{\mathcal{Q}}
        \nonumber
    \end{IEEEeqnarray}
    for all $\*\theta \in \cup_{A \in \sat_{\mathcal{Q}}}A$ so that the conditions of Corollary \ref{corollary:optimal_CL_perfect_memory} are satisfied.
    Notice that one could apply the same logic with most other predictors $f_{\*\theta}$ by replacing $\*\theta^Tx_i^t$ with $f_{\*\theta}(x_i^t)$ in the definition of $\C(\*\theta, \widehat{\mathbb{P}})$ of Example \ref{example:LAD_regression_NPH}.
\end{example}
The take-away message from the previous example is that even though the conditions of Theorem \ref{thm:optimal_CL_perfect_memory} (or Corollary \ref{corollary:optimal_CL_perfect_memory}) will be harder to verify for more complicated model classes, they should be expected to hold in practice unless $\mathcal{Q}$ is substantially restricted and $\C$ is picked very carefully.

\section{Implications for \CL  in the Wild}
\label{sec:consequences_for_CL_design}

Our results are of theoretical interest, but also have two practical implications: 
Firstly, they illuminate that \CL algorithms should be seen as  polynomial time heuristics targeted at solving an \NPH problem.
This new perspective explains why the design of reliable \CL algorithms has proven a persistent challenge.
Secondly, our results provide a theoretically grounded confirmation of recent benchmarking results, which found that \CL algorithms based on experience replay, core sets and episodic memory were more reliable than regularization-based alternatives.
In the remainder of this section, we elaborate on both of these points.

\subsection{\CL as Polynomial Time Heuristics}

As we have shown, \CL algorithms that avoid catastrophic forgetting as judged by an optimality criterion $\mathcal{C}$ generally solve \NPH problems.
Consequently, the polynomial time heuristics that have been proposed to (sub-optimally) tackle the \CL problem in practice can be seen as heuristic algorithms {without} performance guarantees. 
As these heuristics are not coupled to explicit assumptions on the data generating mechanisms underlying the tasks, it is easy to see why reliable \CL algorithms have proven to be a persistent challenge.
Since many well-known \NPH problems admit heuristic polynomial time approximation algorithms \textit{with} performance guarantees, this also raises the question whether one could derive such algorithms for \CL.
So far however, this has not been attempted. 
Instead, the literature has focused on two useful heuristics for the design of \CL algorithms: Memorization and regularization approaches.

\subsection{Memorization versus Regularization}

In the recent large-scale comparative study of \citet{threeScenarios}, replay- and memorization-based \CL heuristics were found to produce far more reliable results than regularization-based approaches.
Similarly, \citet{VariationalCL} found that a variant of their (approximately Bayesian) algorithm which used core sets to represent previously seen tasks produced substantially improved results over the version without core sets.
In the same vein, \citet{GalUnification} found that using generative approaches to complement approximately Bayesian procedures improved performance.
Most recently, \citet{TitsiasCL} outperformed competing approaches using inducing points within the Gaussian Process framework as efficient summaries of previous observations.

%\X{Argument missing here: Often easy to have some guarantee $\C$ on a SINGLE task, that's usually not the problem! So if wee can reduce \CL problem to solving modified single task problem, we might get better results}
In fact, these empirical finding are to be expected given the perfect memory requirement of optimal \CL: Approaches based on replay and core sets amount to storing an approximate representation of previous tasks.
In other words, these \CL algorithms store information $\I_t$ such that one can reconstruct an approximation $\mathbb{Q}_{1:t} \in \mathcal{P}(\X\times\Y)$ for the joint distribution over all tasks observed thus far.
In this sense, it is instructive to think of them as processing a single (albeit consecutively modified) task $\mathbb{Q}_{1:t}$.
Intuitively then, \CL algorithms of this kind will  perform well under two conditions: It must be relatively easy to find some element $\*\theta \in A$ for a \textit{single} $A \in \sat_{\mathcal{Q}}$ and it must hold that $\sat(\mathbb{Q}_{1:t}) \approx \sat_{1:t}$.
The next Example and Figure \ref{Fig:PracticalImplications} expand on this point.
%using linear models and a simple optimality criterion.

\begin{example}
    Keeping things simple, we construct $\mathcal{Q}$ and $\C$ to ensure that $\sat_{\mathcal{Q}}$ consists only of spheres. 
    Suppose that $\*\theta \in \*\Theta = \mathbb{R}^2$ represents the (two-dimensional) mean across all tasks and that $\mathcal{Q}$ consists of empirical measures.
    Further, suppose that $\X = \emptyset$ (so that there only is an output variable $\*Y_t$) and that the criterion of interest is an upper bound of $\varepsilon$ on the average mean squared Euclidean distance, i.e. 
    \begin{IEEEeqnarray}{rCl}
        \C(\*\theta, \widehat{\mathbb{P}}) & = &
        \begin{cases}
            1 & \text{if } \frac{1}{n}\sum_{i=1}^n\|y_i^t - \*\theta\|_2^2 \leq \varepsilon \\
            0 & \text{otherwise}.
        \end{cases}
        \nonumber
    \end{IEEEeqnarray}
    For a single task it is easy to find a value of $\*\theta$ satisfying $\C$ by using simple linear regression (provided that $\varepsilon$ is chosen large enough).
    Figure \ref{Fig:PracticalImplications} illustrates why \CL based on core sets, replay or memory can typically be expected to work relatively well.
\end{example}

\begin{figure}[t!]
        %\vskip 0.1in
        \begin{center}
            \centerline{\includegraphics[trim= {0cm 0.0cm 0cm 0.0cm}, clip, %{1.5cm 3cm 2.2cm 3.6cm},clip, 
            width=1.00\columnwidth]{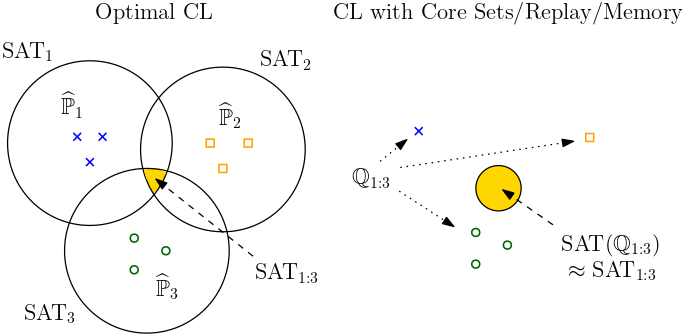}}
            \caption{
            \textbf{Left:} Optimal \CL finds an element $\*\theta_t \in \sat_{1:t}$. \textbf{Right:} \CL algorithms based on core sets/replay/memory find an element $\*\theta_t \in \sat(\mathbb{Q}_{1:t})$, which typically is sufficiently similar to $\sat_{1:t}$.
            }
            \label{Fig:PracticalImplications}
        \end{center}
        \vskip -0.3in
\end{figure}

%\X{Rest needs solid rewrite}

%
%This works because $\widetilde{\sat}_{t} \approx \sat_{1:t}$, where $\widetilde{\sat}_{t} = \sat(\widetilde{\mathbb{P}}_t)$ denotes the set of optimal solutions on the distribution $\widetilde{\mathbb{P}}_t$ that was obtained by merging $\widehat{\mathbb{P}}_t$ and $\{\mathbb{Q}_i\}_{i=1}^t$ into a single distribution.
%
%So long as the \CL training procedure applied to $\widetilde{\mathbb{P}}_t$ can find an element in $\widetilde{\sat}_{t}$, this will generally produce more reliable results than regularization-based methods.

In contrast to memorization-based heuristics, regularization-based approaches have to make  inappropriate assumptions about the difficulty of moving from $\sat_{1:(t-1)}$ to $\sat_{1:t}$.
Specifically, the choice of regularizer corresponds to an implicit and strong assumption on the geometry and nature of overlapping regions between $\sat_{1:(t-1)}$ and $\sat_{t}$.
%
%Typically, these assumptions are based on second-order information
%
As this implicit assumption is usually severely violated in practice, regularization-based \CL algorithms often underperform, especially when the number of tasks is very large \citep{threeScenarios, FineGrainedCL}.
% \subsection{Informing new \CL algorithms}

% While it would in principle be possible to design \CL algorithms with optimality guarantees by sufficiently restricting $\mathcal{Q}$ and $\C$, research on \CL is more interested in algorithms that continue to be reliable without needing explicit assumptions about data generating mechanisms.
% %
% Thus, it is the authors' view that the most promising approaches to \CL in the near future will be based on seeking to parsimoniously approximate previously seen tasks.
%
%While we believe that such approaches are most likely to succeed, there is no free lunch: The better the approximation to previously seen tasks becomes, the larger the memory requirement of the \CL algorithm.

\section{Conclusion}

With this paper, we have produced the first generic theoretical study of the Continual Learning (\CL) problem. 
We did so by translating the notion of catastrophic forgetting into the language of basic set theory.
With this in hand, we showed that  optimal \CL is generally \NPH and requires perfect memory of the past.
This has two practical ramifications: Firstly, it illustrates that existing \CL algorithms can be seen as polynomial time heuristics targeted at solving an \NPH problem.
Secondly, it  reveals why memorization-based \CL approaches using experience replay, core sets or episodic memory have generally proven more successful than their regularization-based alternatives.

\section*{Acknowledgements}

We thank Andreas Damianou and Shuai Tang for fruitful discussions; Isak Falk, Ollie Hammelijnck for spotting a number of typos in the manuscript and Hans Kersting for the gym tour.
JK is funded by EPSRC grant EP/L016710/1 as part of the Oxford-Warwick Statistics Programme (OxWaSP) as well as by the Facebook Fellowship Programme. 
JK is also supported by the Lloyds Register Foundation programme on Data Centric Engineering through the London Air Quality project and by The Alan Turing Institute for Data Science and AI under EPSRC grant EP/N510129/1 in collaboration with the Greater London Authority.

\appendix

\section{Proofs: Optimal \CL is \NPH}

The main results relating to the computational hardness of \CL which are not proven in the main paper are carefully derived in this supplement.

\subsection{Relationship with other problems}

Notice that at each step, the functions $\A_{\*\theta}$ and $\A_{\*I}$ of \OICL define an \textit{optimization problem}, as we are tasked with a parameter value $\*\theta_t$ at each iteration satisfying the specified criterion on all previous tasks. 

\begin{definition}
     Given a fixed 
     hypothesis class $\F_{\*\Theta}$,
     criterion $\C$,
     a set $A \in \sat_{\mathcal{Q}}$
     and $B \in \sat_{\cap}$ for
     \begin{IEEEeqnarray}{rCl}
         \sat_{\cap} & = & \{ \cap_{i=1}^t A_i: A_i \in \sat_{\mathcal{Q}}, \nonumber \\
         &&                 \quad 1 \leq i \leq t \text{ and } 1\leq t \leq T, \:\: T\in\mathbb{N}\},
         \nonumber
     \end{IEEEeqnarray}  
     the \textbf{\OICL optimization problem} is to find a $\*\theta \in A \cap B$.
     Accordingly, the \textbf{\OICL decision problem} is to decide if a solution exists, i.e. if $\*\theta \in A \cap B \neq \emptyset$.
     %
    %  \end{IEEEeqnarray}
    \label{def:OICL_optimization_decision_problem}
\end{definition}

We first show that the \OICL optimization problem is at least as hard as its corresponding decision problem. 

\begin{proposition}
    If one can solve the \OICL optimization problem, one can solve the \OICL decision problem.
    \label{proposition:OICL_opt_solves_OICL_dec}
\end{proposition}
\begin{proof}
    if the \OICL optimization problem can be solved, then there exists some function $f: \sat_{\mathcal{Q}} \times \sat_{\cap} \to \*\Theta$ such that 
    \begin{IEEEeqnarray}{rCl}
        f(A,B) & = & \*\theta 
        \nonumber
    \end{IEEEeqnarray}
    such that $\*\theta \in A\cap B$, for any $A$ and $B$ as in Definition \ref{def:OICL_optimization_decision_problem}.
    But then, one can construct the indicator function
    \begin{IEEEeqnarray}{rCl}
        1(A,B) & = &
        \begin{cases}
            1 & \text{if } f(A,B) \notin \emptyset \\
            0 & \text{otherwise},
        \end{cases}
        \nonumber
    \end{IEEEeqnarray}
    which clearly solves the decision problem.
\end{proof}

The interpretation of this result is clear: Computationally, the \OICL optimization problem is at least as hard as the \OICL decision problem.
This insight is useful mainly because of the next proposition, which shows that any optimal \CL algorithm can solve the \OICL optimization problem.
\begin{proposition}
    If a \CL algorithm is optimal, then it can solve the \OICL optimization problem.
    \label{proposition:OCL_solves_OICL_opt}
\end{proposition}
\begin{proof}
    By Lemma 1, %\ref{lemma:OCL=OICL}, 
    it suffices to show this for an optimal \ICL algorithm.
    Suppose that $\A_{\*\theta}$, $\A_{\*\I}$ define an \OICL algorithm.
    For a given problem instance of the \OICL optimization problem with $A \in \sat_{\mathcal{Q}}$ and $ \cap_{t=1}^{T-1} B_t = B \in \sat_{\cap}$ such that $B_t \in \sat_{\mathcal{Q}}$ for all $1\leq t \leq (T-1)$, we can use the \OICL algorithm to solve the problem.
    To see this, construct the problem instance as $\sat_i = B_i$ for $1\leq i \leq (T-1)$ and $\sat_T = A$. 
    Clearly, the value $\*\theta_T$ generated by the \OICL algorithm satisfies that $\*\theta_T \in A\cap B$. 
\end{proof}
Again, this has a clear interpretation: Computationally, the \OICL algorithm is at least as hard as the \OICL optimization problem.

 \subsection{Proof of Lemma 2}
 
A straightforward Corollary follows. Rewriting this result, one obtains Lemma 2. 

\begin{corollary}
    If a \CL algorithm is optimal, then it can solve the \OICL decision problem.

\end{corollary}
\begin{proof}
    Combine Propositions
    \ref{proposition:OICL_opt_solves_OICL_dec} and \ref{proposition:OCL_solves_OICL_opt}. 
\end{proof}

\subsection{A further refinement}

Indeed, the connection between an optimal \CL algorithm and the \OICL optimization and decision problems can be made much tighter.
The below alternative Lemma could be used in the proof of Theorem 2 to show that optimal \CL does not only solve an \NPH problem, it solves an \NPH problem at \textit{each iteration}.
We do not discuss this in the main paper as the consequences remain the same, though at the expense of complicating the presentation.

\begin{lemma}
    A \CL algorithm is optimal if and only if it solves $T$ \OICL optimization problem instances given by $\{(A_t, B_t)\}_{t=1}^T$ with $A_t = \sat_t = \sat(\widehat{\mathbb{P}}_t)$ and $B_t = \sat_{1:(t-1)} = \cap_{i=1}^{t-1} \sat_i$.
    Similarly, a \CL algorithm is optimal only if it can be used to solve the collection of $T$ \OICL decision problems corresponding to $\{(A_t, B_t)\}_{t=1}^T$.
    \label{proposition:OICL_solves_sequence_of_opt_and_decs}
\end{lemma}
\begin{proof}
    Using Lemma 1 %\ref{lemma:OCL=OICL} 
    once again, it suffices to prove this for an \OICL algorithm.
    The first part of the proposition then follows by definition, as $\*\theta_t \in \sat_{1:t} = \sat_t \cap \left(\cap_{i=1}^{t-1}\sat_i\right)$. Setting  $B = \cap_{i=1}^{t-1}\sat_i$ and $A = \sat_t$ reveals this to be an \OICL optimization problem. 
    The second part follows by combining the first part with the same arguments used in the proof of Proposition \ref{proposition:OICL_opt_solves_OICL_dec}.
\end{proof}

\subsection{Proof of Theorem 1}

With Lemma 2 in place, the proof of Theorem 1 follows by relatively simple arguments that we summarize in two separate propositions.

\begin{proposition}
    If $\mathcal{Q}$ and $\C$ are such that $\sat_{\mathcal{Q}} \supseteq S$ or $\sat_{\cap} \supseteq S$
    so that $S$ is the set of tropical hypersurfaces  or the set of polytopes on $\*\Theta$, 
    then the optimal Idealized \CL decision problem is \NPC.
    \label{proposition:proof_thm_1_1}
\end{proposition}
\begin{proof}
    This is a simple application of the results in \citet{AlgebraicVarieties}  for the case where $S$ is the set of tropical hypersurfaces.
    For the case where $S$ is the set of polytopes, it is a simple application of the results in  \citet{polyhedralPaper} and \citet{polyhedralThesis}.
\end{proof}

\begin{proposition}
    The optimal idealized \CL optimization problem is \NPH.
    \label{proposition:proof_thm_1_2}
\end{proposition}
\begin{proof}
    By proposition \ref{proposition:OICL_opt_solves_OICL_dec}, the idealized \CL optimization problem is at least as hard as the idealized \CL decision problem.
\end{proof}

With this in hand, the proof of Theorem 1 is readily obtained by combining Lemma 2 with Propositions \ref{proposition:OCL_solves_OICL_opt}, \ref{proposition:proof_thm_1_1} and \ref{proposition:proof_thm_1_2}.

\begin{proof}
    First, use Lemma 2: %\ref{lemma:decision_problem}: 
    Optimal \CL can correctly decide if $A\cap B = \emptyset$, for all $A \in \sat_{\cap}$ and $B \in \sat_{\mathcal{Q}}$. 
    %This means that it can be used to decide for any $A \in \sat_{\cap}$ and $B \in \sat_{\mathcal{Q}}$ whether  $A\cap B = \emptyset$.
    %
    %By standard arguments, this can be used to show that this is possible only if they 
    Second, use Proposition \ref{proposition:proof_thm_1_1} to conclude that this decision problem is \NPC. 
    Third, it follows by Proposition \ref{proposition:proof_thm_1_2} that the optimization problem corresponding to an \NPC decision problem is \NPH. 
    Thus, by Proposition \ref{proposition:OCL_solves_OICL_opt}, the result follows.
    %to conclude that deciding whether the intersection $A \cap B$ is empty will be in $\NP$.
\end{proof}

\subsection{Proof of Corollary 1}
    
    It is straightforward to generalize the results of Theorem 1 for all collections $\sat_{\mathcal{Q}}$ whose intersections are as hard to compute as polytopes.

    \begin{proof} 
        Re-use the proof of Theorem 1 and note that if the decision problem  is at least as hard as for polytopes, then the computational complexity of derived as a result of Theorem 1 provides a lower bound.
    \end{proof}

\section{Proofs: Optimal \CL has Perfect Memory}

Next, we show give detailed derivations for the perfect memory result in the main paper.

\subsection{Proof of Lemma 3}

For convenience, we compile the results in Lemma 3 into two separate Propositions.

\begin{proposition}
    $\*\theta' \in \E(\*\theta) \Longleftrightarrow \E(\*\theta) = \E(\*\theta')$. Further, whenever $\*\theta' \notin \E(\*\theta)$, it hold that $\E(\*\theta) \cap \E(\*\theta') = \emptyset$. 
    \label{proposition:equivalence_set_properties}
\end{proposition}
\begin{proof}
    Suppose that $\*\theta' \in \E(\*\theta)$. From the definition of $\E(\*\theta)$, this immediately implies that any $A \in S(\*\theta)$ contains $\*\theta'$. In other words, $\*\theta' \in A \Longleftrightarrow \*\theta \in A$ for all $A \in \sat_{\mathcal{Q}}$. 
    From this, it immediately follows that $S(\*\theta) = S(\*\theta')$ so that
    \begin{IEEEeqnarray}{rCl}
        \E(\*\theta) 
        =
        \bigcap_{
                A \in S(\*\theta)
        }A
        =
        \bigcap_{
                A \in S(\*\theta')
        }A
        = \E(\*\theta'),
        \nonumber
    \end{IEEEeqnarray}
    which proves the first claim.
    The second claim then follows by contradiction: Suppose there was a point $\widetilde{\*\theta}$ such that $\widetilde{\*\theta} \in \E(\*\theta)  \cap \E(\*\theta')$.
    But then, $\widetilde{\*\theta}  \in \E(\*\theta)$, which by the first claim would imply that $S(\*\theta) = S(\widetilde{\*\theta}) = S(\*\theta')$ so that
    \begin{IEEEeqnarray}{rCl}
        \E(\*\theta) 
        =
        \bigcap_{
                A \in S(\*\theta)
        }A
        =
        \bigcap_{
                A \in S(\widetilde{\*\theta})
        }A
        =
        \bigcap_{
                A \in S(\*\theta')
        }A
        = \E(\*\theta').
        \nonumber
    \end{IEEEeqnarray}
    %
    %which clearly yields a contradiction 
    But since $\*\theta' \notin \E(\*\theta)$ and $\*\theta' \in \E(\*\theta')$ it holds that $\E(\*\theta) \neq \E(\*\theta')$, which yields the desired contradiction.
\end{proof}

\begin{proposition}
    For all $A \in \sat_{\mathcal{Q}}$ and all $\*\theta \in \*\Theta$, either $\E(\*\theta) \subseteq A$ or $\E(\*\theta) \cap A = \emptyset$.
\end{proposition}
\begin{proof}
    This follows by definition of $\E(\*\theta)$: Either $A \in S(\*\theta)$, in which case it must follow that $\E(\*\theta) \subseteq A$. Alternatively, if $A \notin S(\*\theta)$, then by Proposition \ref{proposition:equivalence_set_properties} one has that $A \notin S(\*\theta')$ for any $\*\theta' \in \E(\*\theta)$, which means that $\E(\*\theta) \cap A = \emptyset$. 
\end{proof}

\subsection{Proof of Lemma 4}

We first define the notion of a Decision Problem Oracle set.
Note that Lemma 4 in the main paper revolves around such a Decision Problem Oracle set (albeit without using this name).

\begin{definition}
     Given a set $\sat_{1:t} \in \sat_{\cap}$, a set $C$ is a \textbf{Decision Problem Oracle set} for $\sat_{1:t}$ if
     \begin{IEEEeqnarray}{rCl}
        C \cap A = \emptyset \Longleftrightarrow
        \sat_{1:t} \cap A = \emptyset, \nonumber
     \end{IEEEeqnarray}
     for any $A \in \sat_{\mathcal{Q}}$.
\end{definition}

\begin{proposition}
    If a \CL algorithm is optimal, there exists a function $h: \*\Theta \times \Info \to 2^{\*\Theta}$ such that for $\*\theta_t, \*\I_t$ as in Definition 6, %\ref{def:ICL}, 
    $C_t = h(\*\theta_t, \*\I_t)$ 
    is a Decision Problem Oracle set for $\sat_{1:t}$, and this holds for all $1\leq t \leq T$.
    \label{proposition:C_t_exists}
\end{proposition}
\begin{proof}
    If the \CL algorithm is optimal, it can solve the \OICL decision problem given by $A = \sat_{t+1}$ and $B = \sat_{1:t}$ at the $(t+1)$-th task. 
    (This follows by combining Lemma 1 %\ref{lemma:OCL=OICL} 
    with Proposition \ref{proposition:OICL_solves_sequence_of_opt_and_decs})
    Specifically, because
    \begin{IEEEeqnarray}{rCl}
        \A(\*\theta_t, \*\I_t, A) \notin \emptyset
        & \Longleftrightarrow &
        \sat_{1:t} \cap A \neq \emptyset,
        \nonumber 
    \end{IEEEeqnarray}
    it is clear that there must exist a function $g: \*\Theta \times \Info \times \sat_{\mathcal{Q}} \to 2^{\*\Theta}$ for which $C_t = g(\*\theta_t, \*\I_t, A)$ is such that
    \begin{IEEEeqnarray}{rCl}
        C_t \cap A \neq \emptyset
        & \Longleftrightarrow &
        \sat_{1:t} \cap A \neq \emptyset.
        \nonumber 
    \end{IEEEeqnarray}
    Furthermore, it is clear that $g$ will be constant in $A$ (since $\sat_{1:t}$ is), so that one can write $C_t = h(\*\theta_t, \*I_t)$ for some $h:\*\Theta \times \Info \to 2^{\*\Theta}$ instead.
\end{proof}

\subsection{Assumptions for the Decision Problem Oracle set}

We use the observation of the last subsection to investigate the memory requirements of optimal \CL algorithms.
Before doing so, we first make some assumptions that are useful for proving Theorem 2 and Corollary 2.

\subsubsection{Assumptions for Theorem 2}

\begin{assumption}[Storage efficiency]
    $C_t \subseteq \sat_{1:t}$
    \label{assumption:storage_efficiency}
\end{assumption}
\begin{assumption}[Information efficiency]
    $C_t \subseteq C_{t-1}$
    \label{assumption:information_efficiency}
\end{assumption}
\begin{assumption}[Finite identifiability]
    There exists a finite sequence of sets $\{A_t\}_{t=1}^T$ in $\sat_{\mathcal{Q}}$ such that $\cap_{t=1}^TA_t = \E(\*\theta)$, for all $\*\theta \in \*\Theta$.
    \label{assumption:finite_identifiability}
\end{assumption}

\begin{remark}
    Assumption \ref{assumption:storage_efficiency} ensures that $C_t$ takes up as little space in memory as possible. 
    To illustrate this, suppose that $\widetilde{C}_t \cap \sat_{t+1}  = \emptyset \Longleftrightarrow \sat_{1:t} \cap \sat_{t+1} = \emptyset$, but that $\widetilde{C}_t \setminus \sat_{1:t} \neq \emptyset$.
    In this case, it clearly holds for $C_t = \widetilde{C}_t \cap \sat_{1:t} \subset \widetilde{C}_t$ that $C_t \cap \sat_{t+1}  = \emptyset \Longleftrightarrow \sat_{1:t} \cap \sat_{t+1} = \emptyset$, too.
    In other words, one can construct an alternative and stritly smaller Decision Problem Oracle set $C_t$ from $\widetilde{C}_t$ by removing all points that are not also in $\sat_{1:t}$.
\end{remark}
\begin{remark}
    Assumption \ref{assumption:information_efficiency} ensures that the algorithm learns monotonically. Specifically, %one may think about $C_t$ as a coarse approximation to $\sat_{1:t}$.  
    it ensures that each additional task will shrink the set $\sat_{1:t}$ of parameter values that satisfy the criterion $C$ on all task $1,2,\dots t$.
    %Reasoning along these lines, it is clear why it is desirable that $C_t$ be non-increasing with an increasing number of tasks. 
    This is intuitively appealing since it means that the algorithm never incorrectly discards a parameter only to add it back in at a later task.
\end{remark}
\begin{remark}
    Assumption \ref{assumption:finite_identifiability}
    says that equivalence sets are reachable with finitely many tasks. 
    In other words, there exist collections of tasks which satisfy the algorithm's optimality criterion $C$ only if the parameter that is learnt lies in a single equivalence set.
    %
    %This is not strictly needed, but simplifies the proofs drastically.
    %
\end{remark}

\subsubsection{Assumptions for Corollary 2}

As we shall see shortly, if we strengthen Assumption \ref{assumption:finite_identifiability}, we can drastically simplify the proof of Theorem 2 and drop the other two assumptions required for the result.

\begin{assumption}[Identifiability]
    $\E(\*\theta) \in \sat_{\mathcal{Q}}$, for all $\*\theta \in \*\Theta$.
    \label{assumption:identifiability}
\end{assumption}
\begin{remark}
    Simply put, this means that each equivalence set can be ``reached'' with a single task. In other words, each equivalence set is identifiable with a single task.
\end{remark}

\subsection{Proof of Theorem 2}

Notice that proving Theorem 2 is equivalent to proving the proposition below.

\begin{proposition}
    Under Assumptions \ref{assumption:storage_efficiency}, \ref{assumption:information_efficiency} and \ref{assumption:finite_identifiability}, any optimal \CL algorithm has perfect memory.
    \label{proposition:perfect_memory1}
\end{proposition}
\begin{proof}
    We show this by proving that for some arbitrary minimal representation $\{f(i)\}_{i\in I}$ and $F = \cup_{i\in I}\{f(i)\}$, $C_t \supseteq F \cap \sat_{1:t}$. 

    First, we show that $\widetilde{C}_t = F \cap \sat_{1:t}$ is a Decision Problem Oracle set. In other words, we show that $\widetilde{C}_t \cap A = \emptyset \Longleftrightarrow \sat_{1:t} \cap A = \emptyset$, for all $A \in \sat_{\mathcal{Q}}$ and any $\sat_{1:t} \in \sat_{\cap}$. We do so by contradiction:
    Suppose that $\exists A \in \sat_{\mathcal{Q}}$ so that $\widetilde{C}_t \cap A = \emptyset$, but $\sat_{1:t} \cap A\neq \emptyset$. 
    But then, $A \cap \sat_{1:t}$ contains at least one point, say $\*\theta$. 
    By construction of $\widetilde{C}_t$, it also follows that $F\cap A = \emptyset$.
    This yields the desired contradiction, since by virtue of $A \supseteq E(\*\theta)$ it also implies that $F \cap E(\*\theta) = \emptyset$, even though $F$ contains exactly one point for each equivalence set by definition, including the equivalence set $E(\*\theta)$.
    In other words, it is \textit{sufficient} for the optimal \CL algorithm to be able to reconstruct the Decision Problem Oracle set $\widetilde{C}_t = F \cap \sat_{1:t}$ at task $(t+1)$.

    Second, we demonstrate that this is also \textit{necessary}: Suppose that there exists some $\*\theta \in \widetilde{C}_t$ such that $\widetilde{C}_t \setminus \{\*\theta\}$ is also a Decision Problem Oracle set, for all $1\leq t \leq T$. 
    By virtue of Assumption \ref{assumption:finite_identifiability}, we can construct a finite sequence of sets $\{A_i\}_{i=t+1}^T$ such that $A_i \in \sat_{\mathcal{Q}}$ and $\cap_{i=t+1}^TA_i = E(\*\theta)$. 
    By construction, $\sat_{1:t} \cap \left(\cap_{i=t+1}^TA_i\right) \neq \emptyset$, but $\left(\widetilde{C}_t \setminus \{\*\theta\}\right) \cap \left(\cap_{i=t+1}^TA_i\right) = \emptyset$. 
    Since it also holds that $C_t \subseteq C_{t-1}$, it follows that $\left(\widetilde{C}_{T-1} \setminus \{\*\theta\}\right) \cap A_T = \emptyset$, which completes the proof.
\end{proof}

\subsection{Proof of Corollary 2}

Alternatively, one could drop the first two assumptions and strengthen the third to draw the same conclusion.  

\begin{proposition}
    Under Assumption \ref{assumption:identifiability}, the optimal \CL algorithm has perfect memory.
    \label{proposition:perfect_memory2}
\end{proposition}
\begin{proof}
    The proof of sufficiency is exactly equal to the one in Proposition \ref{proposition:perfect_memory1}.
    The proof of necessity follows along the same lines as before but is even easier: Since one can always select $A = \E(\*\theta)$, $\sat_{1:t} \supseteq C_t \supseteq \sat_{1:t} \cap F$ readily follows. (Indeed, it follows that $C_t = \sat_{1:t}$ because $\sat_{1:t} \cap F$ for any $F = \cup_{i\in I}\{f(i)\}$ generated through a Minimal Representation $\{f(i)\}_{i\in I}$.)
\end{proof}

\bibliography{library}
\bibliographystyle{icml2019}

\end{document}